\documentclass[twoside,11pt]{article}
\usepackage{jmlr2e}
\usepackage[all,cmtip]{xy} 
\usepackage{color,hyperref}
\hypersetup{colorlinks,breaklinks,linkcolor=darkblue,urlcolor=darkblue, anchorcolor=darkblue,citecolor=darkblue}
\usepackage{amsmath}
\usepackage{mathtools}
\usepackage{enumerate}
\newtheorem{Thm}{Theorem}
\definecolor{darkblue}{rgb}{0.0,0.0,0.5}

\newcommand{\phibp}{\phi_{ \hspace{-0.025in}\scalebox{.45}{\text{ BP}}}}
\newcommand{\phics}{\phi_{ \hspace{-0.025in}\scalebox{.45}{\text{ CS}}}}
\newcommand{\lone}{$\ell_1$-norm }

\DeclareMathOperator{\dom}{\bf Dom}

\DeclarePairedDelimiter\onenorm{\lVert}{\rVert_1}
\DeclarePairedDelimiter\znorm{\lVert}{\rVert_0}

 \def\0{{\bf 0}}

\def\qed{\hfill\hbox{${\vcenter{\vbox{
    \hrule height 0.4pt\hbox{\vrule width 0.4pt height 6pt
    \kern5pt\vrule width 0.4pt}\hrule height 0.4pt}}}$}}

\definecolor{myred}{rgb}{0.3,0.0,0.7}
\definecolor{dkg}{rgb}{0.1,0.7,0.2}
\definecolor{dkb}{rgb}{0.0,0.2,0.8}

\def\bfone{{\mathbf{1}}}

\def\bfu{{\mathbf u}}
\def\bfv{{\mathbf v}}

\def\Dc{{\cal D}}

\def\Hc{{\cal H}}

\def\Lc{{\cal L}}

\def\Rbb{{\mathbb R}}

\newcommand{\bprf}{\begin{proof}}
\newcommand{\eprf}{\end{proof}}
\newcommand{\bt}{\begin{theorem}}
\newcommand{\et}{\end{theorem}}
\def\beq{\begin{equation}\begin{aligned}}
\def\eeq{\end{aligned}\end{equation}\noindent}
\def\beqq{\begin{equation*}\begin{aligned}}
\def\eeqq{\end{aligned}\end{equation*}\noindent}
\def\beqn{\begin{eqnarray}}
\def\eeqn{\end{eqnarray} \noindent}

\jmlrheading{?}{2015}{?}{2/15}{??/??}{Hamed~Masnadi-Shirazi, Nuno~Vasconcelos and Arya~Iranmehr}
\ShortHeadings{Cost-sensitive Support Vector Machines}{Masnadi-Shirazi, Vasconcelos and Iranmehr}
\firstpageno{1}

\begin{document}

\title{Cost-sensitive Support Vector Machines}

\author{\name Hamed~Masnadi-Shirazi\email hmasnadi@shirazu.ac.ir\\
				\addr Department of Electrical Engineering, Shiraz University, Shiraz, Iran\\
       \name Nuno~Vasconcelos\email nuno@ucsd.edu\\
       \addr Statistical Visual Computing Laboratory, University of California, San Diego, La Jolla, CA~~92039\\
       \name Arya~Iranmehr \email airanmehr@ucsd.edu \\
       \addr Department of Electrical and Computer Engineering, University of California, San Diego, La Jolla, CA~~92039\\
			}

\editor{}

\maketitle

\begin{abstract}
A new procedure for learning cost-sensitive SVM(CS-SVM) classifiers is
proposed. The SVM hinge loss is extended to the cost sensitive
setting, and the CS-SVM is derived as the minimizer
of the associated risk. The extension of the hinge loss 
draws on recent connections between risk minimization and 
probability elicitation. These connections are generalized to
cost-sensitive classification, in a manner that guarantees
consistency with the cost-sensitive Bayes risk, and 
associated Bayes decision rule. This ensures that optimal decision
rules, under the new hinge loss, implement the Bayes-optimal cost-sensitive 
classification boundary.  Minimization of the new hinge loss is shown to
be a generalization of the classic SVM optimization problem, and can be 
solved by identical procedures. 
The dual problem of CS-SVM is carefully scrutinized by means of regularization theory and sensitivity analysis and the CS-SVM algorithm is substantiated.
The proposed algorithm is also extended to  cost-sensitive learning with example dependent costs.
The minimum cost sensitive risk is proposed as the performance measure and is connected to ROC analysis through vector optimization. 
The resulting algorithm avoids the
shortcomings of previous approaches to cost-sensitive SVM design, and is shown to have
superior experimental performance on a large number of cost sensitive and imbalanced datasets.
\end{abstract}


\section{Introduction}
The most popular strategy for the design of classification algorithms 
is to minimize the probability of error, assuming that all misclassifications
have the same cost. The resulting decision rules are usually denoted
as {\it cost-insensitive\/}. However, in many important applications of machine
learning, such as medical diagnosis, fraud detection, or business decision 
making, certain types of error are much more costly than others.
Other applications involve significantly unbalanced datasets,
where examples from different classes appear with substantially different 
probability. It is well known, from Bayesian decision theory, that under
any of these two situations (uneven costs or probabilities), the 
optimal decision rule deviates from the optimal cost-insensitive rule 
in the same manner. In both cases, reliance on cost insensitive 
algorithms for classifier design can be highly sub-optimal. While
this makes it obviously important to develop {\it cost-sensitive\/}
extensions of state-of-the-art machine learning techniques, the current
understanding of such extensions is limited.

In this work we consider the support vector machine (SVM) architecture \Citet{SVN}.
Although SVMs are based on a very solid learning-theoretic foundation, and 
have been successfully applied to many classification problems,  
it is not well understood how to design cost-sensitive extensions of
the SVM learning algorithm. The standard, or cost-insensitive, SVM
is based on the minimization of a symmetric loss function (the hinge loss)
that does not have an obvious cost-sensitive generalization. 
In the literature, this problem has been addressed by various
approaches, which can be grouped into three general categories. The first
is to address the problem as one of data processing, by adopting
resampling techniques that under-sample the majority class 
and/or over-sample the minority class
\Citet{CurseImbal,SVM-SMOTE,SVM-Apply,PerceptSVM,ED-03}.
Resampling is not easy when the classification unbalance is due to
either different misclassification costs (not clear what the class
probabilities should be) or an extreme unbalance in class 
probabilities (sample starvation for classes of 
very low probability). 
It also does not guarantee that the learned SVM will change, since it
could have no effect on the support vectors. 
Active learning based methods have also been proposed to train the SVM algorithm on the informative instances, instances which are close to the hyperplane \Citet{Active-Imbalanced}.

The second class of approaches
\Citet{KernelMod,SVM-KBA2,SVM-KBA1}
involve kernel modifications. These methods are based on conformal 
transformations of the input or feature space, by modifying the 
kernel used by the SVM. They are somewhat unsatisfactory, due to the 
implicit assumption that a linear SVM cannot be made cost-sensitive. It 
is unclear why this should be the case.

The third, and most widely researched, approach is to modify the SVM 
algorithm in order to achieve cost sensitivity. This is done in one
of two ways. The first is a naive method, known as {\it boundary movement (BM-SVM),\/} 
which shifts the decision boundary by simply adjusting the threshold 
of the standard SVM \Citet{SVM-Bnaive}. Under Bayesian decision theory,
this would be the optimal strategy if the class posterior probabilities
were available. However, it is well known that SVMs do not predict 
these probabilities accurately. 
While a literature has developed in the area of probability 
calibration~\Citet{platt}, calibration techniques do not
aid the cost-sensitive performance of threshold manipulation. This
follows from the fact that all calibration techniques rely on
an invertible (monotonic and one-to-one) transformation of the SVM 
output. 
Because the manipulation of a threshold at either the input or 
output of such a transformation produces the same 
receiver-operating-characteristic (ROC) curve, calibration does not
change cost-sensitive classification performance. The boundary movement method
is also obviously flawed when the data is non-separable, in which
case cost-sensitive optimality is expected to require a modification
of {\it both} the normal of the separating plane $w$ and the classifier 
threshold $b$. The second proposal to modify SVM learning is known as 
the {\it biased penalties (BP-SVM) \/} method
\Citet{SVM-Asym,SVM-NonStand,SVM-Falarm,SVM-Newv,LibSVM}.
This consists of introducing different penalty factors $C_1$ and $C_{-1}$ 
for the positive and negative SVM slack variables during training. 
It is implemented by transforming the primal SVM problem into 
\begin{equation}\label{eq:BPSVM}
\begin{aligned}
& \underset{w,b,\xi}{\text{argmin}} \hspace{0.1in}  && \frac{1}{2} ||w||^2 + C \left[ C_1\sum_{\{i|y_i=1\}}\xi_i + C_{-1}\sum_{\{i|y_i=-1\}}\xi_i \right]  \\
& \text{subject to}    &&  y_i(w^Tx+b) \geq 1-\xi_i. 
\end{aligned}
\end{equation}
The biased penalties method also suffers from an obvious flaw, which is
converse to that of the boundary movement method: it has limited ability to 
enforce cost-sensitivity when the training data is separable. For large slack penalty $C$,
the slack variables $\xi_i$ are zero-valued and the optimization above 
degenerates into that of the standard SVM, where the decision boundary is 
placed midway between the two classes rather than assigning a larger 
margin to one of them. 

In this work we propose an alternative strategy for the design of 
cost-sensitive SVMs. This strategy is fundamentally different 
from previous attempts, in the sense that is does not directly manipulate 
the standard SVM learning algorithm. Instead, we extend the SVM hinge loss,
and derive the optimal cost-sensitive learning algorithm as the minimizer
of the associated risk. The derivation of the new cost-sensitive hinge loss
draws on recent connections between risk minimization and 
probability elicitation~\Citet{savage-loss}. Such connections are 
generalized to the case of cost-sensitive classification.

It is shown that it is always possible to specify the 
predictor and conditional risk functions desired for
the SVM classifier, and derive the loss for which these are optimal. 
A sufficient condition for the cost-sensitive Bayes-optimality of
the predictor is then provided, as well as necessary conditions for conditional
risks that approximate the cost-sensitive Bayes risk. Together,
these conditions enable the design of a new hinge loss which is minimized by 
an SVM that 1) implements the cost-sensitive Bayes decision rule, and 
2) approximates the cost-sensitive Bayes risk. It is also shown that the 
minimization of this loss is a generalization of the classic SVM 
optimization problem, and can 
be solved by identical procedures. The resulting algorithm avoids the 
shortcomings of previous methods, producing cost-sensitive decision rules 
for {\it both\/} cases of separable and inseparable training data. 
Experimental results show that these advantages result in better
cost-sensitive classification performance than previous solutions.

Since CS-SVM is implemented in the dual, cost-sensitive learning in the dual should be studied more closely. We show that cost-sensitive learning in the dual appears as regularization and changing the constraint's upper bounds which stem from sensitivity analysis. These connections are considered under cost-sensitive learning and imbalanced data learning.

Moreover, we  show that  in the cost-sensitive  and imbalanced data settings, the priors and costs should be incorporated in the performance measure. We propose minimum expected (cost-sensitive) risk  as a cost sensitive  performance metric and demonstrate its connections to the ROC curve. For the case of unknown costs, we introduce a robust measure which reflects the performance of the classifier under a given tolerance of false-positive or false-negative errors.

The paper is organized as follows. Section 2 briefly reviews the 
probability elicitation view of loss function 
design \Citet{savage-loss}. Section 3 then generalizes
the connections between probability elicitation and risk minimization
to the cost-sensitive setting. In Section 4, these connections 
are used to derive the new SVM loss and algorithm. 
In section 5, the dual problem of CS-SVM is thoroughly evaluated in the sense of regularization and sensitivity analysis.
Section 6 presents an extension of CS-SVM for problems with example-dependent costs.
Section 7 proposes minimum cost sensitive risk as a standard measure for examining classifier performance in the cost-sensitive and imbalanced data setting.
Finally, Section 8 presents an experimental evaluation
that demonstrates improved performance of the proposed cost 
sensitive SVM  over previous methods. 

\section{Bayes consistent classifier design}
The goal of classification is to map feature vectors ${\bf x} \in \cal X$  to class labels $y \in \{-1,1\}$.
From a statistical viewpoint, the feature vectors and class labels are drawn from probability distributions 
$P_{\bf X}({\bf x})$ and $P_Y(y)$ respectively.
In terms of functions, we write a classifier as $h({\bf x}) = sign[p({\bf x})]$, where the function $p: {\cal X} \rightarrow \mathbb{R}$
is denoted as the classifier predictor. Given a non-negative function $L(p({\bf x}),y)$ that assigns a loss to each $(p({\bf x}),y)$ pair, the classifier is considered optimal if it minimizes the expected loss $R = E_{{\bf X},Y}[L(p({\bf x}),y)]$, also known as the risk. Minimizing the risk, is itself equivalent to minimizing the conditional risk 
\begin{eqnarray}
 E_{Y|{\bf X}} [L(p({\bf x}),y)|{\bf X} = {\bf x}] = 
 P_{Y|{\bf X}}(1|{\bf x}) L(p({\bf x}),1)  \nonumber \\ + (1-P_{Y|{\bf X}}(1|{\bf x})) L(p({\bf x}),-1),
\label{eq:Cphi}
\end{eqnarray}
for all ${\bf x} \in {\cal X}$. It is discerning to write the predictor function $p({\bf x})$  as a composition of two functions
$p({\bf x}) = f(\eta({\bf x}))$, 
where $\eta({\bf x}) = P_{Y|{\bf X}}(1|{\bf x})$ is the posterior probability , and
$f: [0,1] \rightarrow \mathbb{R}$ is denoted as the {\it link function}. This provides a valuable connection to the Bayes decision rule. 
A loss is considered Bayes consistent when its associated risk is minimized by the BDR. For example the  zero-one loss can be written as
\begin{eqnarray}
L_{0/1}(f,y) &=& \frac{1- sign(yf)}{2} \nonumber \\ &=& \left\{ \begin{array}{ll}
         0, & \text{if $y=sign(f)$};\\
        1, & \text{if $y \ne sign(f)$},\end{array} \right.
\end{eqnarray}
where we omit the dependence on $\bf x$ for notational simplicity.
The conditional risk for this loss function is 
\begin{eqnarray}
  C_{0/1}(\eta,f) &=& \eta \frac{1- sign(f)}{2} + 
  (1-\eta) \frac{1 + sign(f)}{2}
   \nonumber\\ &=& \left\{ \begin{array}{ll}
         1-\eta, & \text{if $f \geq 0 $};\\
        \eta, & \text{if $f<0$}.\end{array} \right.
\end{eqnarray}
This risk is minimized 
by any predictor $f^*$ such that 
\begin{equation}
  \left\{
  \begin{array}{cc}
    f^*({\bf x}) > 0 & \text{if $\eta({\bf x}) > \gamma $} \\
    f^*({\bf x}) = 0 & \text{if $\eta({\bf x}) =  \gamma $} \\
    f^*({\bf x}) < 0 & \text{if $\eta({\bf x}) <  \gamma $} 
  \end{array}
  \right.
  \label{eq:Bayesnec}
\end{equation}
and $\gamma=\frac{1}{2}$.
Examples of optimal predictors include $f^*=2\eta-1$ and
$f^*=\log\frac{\eta}{1-\eta}$. The associated optimal classifier 
$h^* = sign[f^*]$ is the well known Bayes decision rule thus proving that the zero-one loss is Bayes consistent. Finally, the associated 
minimum conditional (zero-one) risk is
\begin{eqnarray}
  C^*_{0/1} (\eta) = \eta\left(\frac{1}{2}-\frac{1}{2}sign(2\eta-1)\right)+ \nonumber \\
           (1-\eta)\left(\frac{1}{2}+\frac{1}{2}sign(2\eta-1)\right).
\end{eqnarray}

A handful of other losses have been  shown to be Bayes consistent. These include  the exponential loss used in boosting classifiers \cite{friedman}, 
logistic loss of logistic regression \cite{friedman, Zhang04}, or the hinge loss of SVMs \cite{Zhang04}. These losses are 
of the form $L_{\phi}(f,y) = \phi(yf)$ for different functions $\phi(\cdot)$ and are known as {\it margin losses}. Margin losses assign a non-zero penalty to small positive $yf$, encouraging the creation of a margin. The resulting large-margin classifiers
have better generalization than those produced by the zero-one loss or other losses that do not enforce a margin ~\cite{book:STL}.
For a margin loss, the 
conditional risk is simply 
\begin{equation}
  C_\phi(\eta,f) = \eta \phi(f) + (1-\eta) \phi(-f).
  \label{eq:CondRisk}
\end{equation}
The conditional risk is minimized by the predictor
\begin{equation}
  f^*_{\phi}(\eta) = \arg\min_{f} C_\phi(\eta,f)
\end{equation}
and the minimum conditional risk 
 is $C^*_\phi(\eta) = C_\phi(\eta,f^*_\phi)$.

Recently, a generative formula for the derivation of novel Bayes consistent loss functions has been presented in \cite{savage-loss} relying on classical probability elicitation in statistics~\Citet{Savage}. Comparable to risk minimization, in probability elicitation, the goal is to find the probability estimator ${\hat \eta}$ that maximizes the expected reward
\begin{equation}
  I(\eta,{\hat \eta}) = \eta I_{1}({\hat \eta}) + (1-\eta) I_{-1}({\hat \eta}),
  \label{eq:expreward}
\end{equation}
where $I_1({\hat \eta})$ is the reward for predicting ${\hat \eta}$ when 
event $y=1$ holds and $I_{-1}({\hat \eta})$ the corresponding reward when
$y=-1$. The functions $I_1(\cdot), I_{-1}(\cdot)$ are such that the 
expected reward is maximal when ${\hat \eta} = \eta$, i.e. 
\begin{equation}
  I(\eta,{\hat \eta}) \leq I(\eta, \eta) = J(\eta), \,\,\, \forall \eta
  \label{eq:Savagebound}
\end{equation}
with equality if and only if ${\hat \eta} = \eta$. 

\begin{Thm}{~\Citet{Savage}}
  Let $  I(\eta,{\hat \eta})$ and $J(\eta)$ be as defined 
  in~(\ref{eq:expreward}) and (\ref{eq:Savagebound}). Then
  1) $J(\eta)$ is convex and 2)~(\ref{eq:Savagebound}) holds if and only if 
  \begin{eqnarray}
    \label{eq:Is}
    I_1(\eta) &=& J(\eta) + (1-\eta) J^\prime(\eta) \label{eq:I1}  
    \label{eq:I1Jprime} \\
    I_{-1}(\eta) &=& J(\eta) -\eta J^\prime(\eta) \label{eq:I-1}.
    \label{eq:I2Jprime}
  \end{eqnarray}
  \label{thm:savage}
\end{Thm}

The theorem states that $I_1(\cdot), I_{-1}(\cdot)$ can be derived such that ~(\ref{eq:Savagebound}) 
holds by applying an appropriate convex $J(\eta)$. This primary theorem was used in \cite{savage-loss} to establish the following for margin loss functions.  

\begin{Thm}{~\Citet{savage-loss}}
  Let $J(\eta)$ be as defined in (\ref{eq:Savagebound}) and
  $f$ a continuous function. If the following properties hold
  \begin{enumerate}
  \item $J(\eta) = J(1-\eta)$,
  \item $f$ is invertible with symmetry
    \begin{equation}
      f^{-1}(-v) = 1 -  f^{-1}(v),
      \label{eq:linksym}
    \end{equation}
  \end{enumerate}
  then the functions $I_1(\cdot)$ and $I_{-1}(\cdot)$ derived with
  (\ref{eq:I1Jprime}) and (\ref{eq:I2Jprime}) satisfy the following equalities
  \begin{eqnarray}
    I_1(\eta) &=& -\phi(f(\eta)) \label{eq:I1f}\\
    I_{-1}(\eta) &=& -\phi(-f(\eta)) \label{eq:I-1f},
  \end{eqnarray}
  with
  \begin{equation}
    \phi(v) =  -J[f^{-1}(v)] - (1- f^{-1}(v)) J^\prime[f^{-1}(v)].
    \label{eq:phieq}
  \end{equation}
\end{Thm}
This theorem provides a generative path for designing Bayes consistent margin loss functions for classification. 
Specifically, any convex  symmetric function $J(\eta) = -C^*_\phi(\eta)$ and invertible function $f^{-1}$ satisfying (\ref{eq:linksym}) can be used in equation (\ref{eq:phieq}) to derive a novel Bayes consistent loss function $\phi(v)$. This is in contrast to previous approaches which require guessing a loss function $\phi(v)$ and checking that it is Bayes consistent by minimizing  $C_\phi(\eta,f)$, so as to obtain whatever optimal predictor $f^*_\phi$ and minimum expected risk $C^*_\phi(\eta)$ results \cite{Zhang04} or methods that restrict the loss function to being convex, differentiable at zero, and have negative derivative at the origin \cite{BartlettJordanMcAuliffe2006}.

\section{ Cost sensitive Bayes consistent classifier design}
\label{sec:Asy}

In this section we extend the connections between risk minimization
and probability elicitation to the cost-sensitive setting.
We start by reviewing the cost-sensitive zero-one loss.

\subsection{Cost-sensitive zero-one loss}

The cost-sensitive extension of the zero-one loss is
\begin{eqnarray}
  && \!\!\!\!\!\!\!\!\! L_{C_1,C_{-1}}(f,y) =  \nonumber \\ 
  && \!\!\!\!\!\!\!\! \frac{1-sign(yf)}{2}\left(C_1 \frac{1-sign(f)}{2}
    + C_{-1} \frac{1+sign(f)}{2}\right)  \nonumber \\  
  && \!\!\!\!\!\!\!\! =\left\{ \begin{array}{ll}
      0, & \text{if $y=sign(f)$};\\
      C_1, & \text{if $y=1$ and $sign(f) = -1$} \\
      C_{-1}, & \text{if $y=-1$ and $sign(f) = 1$},
    \end{array} \right. 
\end{eqnarray}
where $C_1$ is the cost of a false negative and 
$C_{-1}$ that of a false positive. The associated conditional
risk is 
\begin{eqnarray}
  && \!\!\!\!\!\!\!\!\! C_{C_1,C_{-1}}(\eta,f) = \nonumber \\ 
  && \!\!\!\!\!\!\!\!\! C_1 \eta \frac{1- sign(f)}{2} + 
  (1-\eta) C_{-1} \frac{1 + sign(f)}{2} = \nonumber \\
  && \!\!\!\!\!\!\!\!\! = \left\{ \begin{array}{ll}
         C_{-1} (1-\eta), & \text{if $f \geq 0 $};\\
        C_1 \eta, & \text{if $f<0$},\end{array} \right. 
\end{eqnarray}
and is minimized by any predictor that satisfies (\ref{eq:Bayesnec}) 
with $\gamma=\frac{C_{-1}}{C_1 + C_{-1}}$.
Examples of optimal predictors include $f^*(\eta)=(C_1+C_{-1})\eta-C_{-1}$ 
and $f^*(\eta)=\log\frac{\eta C_1}{(1-\eta)C_{-1}}$. The associated 
optimal classifier $h^* = sign[f^*]$ implements the cost-sensitive Bayes decision 
rule, and the associated minimum conditional (cost-sensitive) 
risk is 
\begin{eqnarray}
  C^*_{C_1,C_{-1}} (\eta) = 
  C_1 \eta\left(\frac{1}{2}-
    \frac{1}{2}sign\left[f^*(\eta)\right]\right)+ \nonumber \\
  C_{-1}(1-\eta)\left(\frac{1}{2}+
    \frac{1}{2}sign\left[f^*(\eta)\right]
  \right)
  \label{eq:01minRisk}
\end{eqnarray}
with $f^*(\eta)=(C_1+C_{-1})\eta-C_{-1}$.
We show that the minimum cost sensitive zero-one risk is equivalent to the {\it minimum cost sensitive Bayes error}.
\begin{Thm}
\label{Thm:MinBayesErr}
The minimum risk associated with the cost sensitive zero-one loss is equal to the minimum cost sensitive Bayes error.
\end{Thm}

\begin{proof}
\begin{eqnarray}
\label{eq:ZeroOneRefinement}
&&R^*_{C_1,C_{-1}}=E_X[C^*_{C_1,C_{-1}}(\eta)]=\int P(x)C^*_{C_1,C_{-1}}(P(1|x)) dx= \\
&&\int_{P(1|x)\geq\gamma} (\frac{P(x|1)+P(x|-1)}{2})(C_{-1}(1-\frac{P(x|1)}{P(x|1)+P(x|-1)})) dx + \\
&&\int_{P(1|x) < \gamma} (\frac{P(x|1)+P(x|-1)}{2})(C_1(\frac{P(x|1)}{P(x|1)+P(x|-1)})) dx =\\
&&\frac{1}{2}\int_{P(1|x)\geq\gamma} C_{-1}P(x|-1) dx +\frac{1}{2}\int_{P(1|x)< \gamma} C_1P(x|1) dx = \\
&&\frac{1}{2}(C_{-1}\epsilon^{\gamma}_1+C_1\epsilon^{\gamma}_2)=\epsilon_{C_1,C_{-1}}
\end{eqnarray}
where $\epsilon^{\gamma}_2$ and $\epsilon^{\gamma}_1$ are the miss rate and false positive rate  associated with the cost sensitive threshold $\gamma$ and $\epsilon_{C_1,C_{-1}}$ is the cost sensitive Bayes error rate. We have also assumed, without loss of generality, that the prior probabilities are equal. 
\end{proof}

The next theorem highlights some fundamental properties of the minimum conditional cost-sensitive zero-one 
risk.
\begin{Thm}
\label{Thm:RiskSymmetryProps}
The risk of (\ref{eq:01minRisk}) 
has the following properties:
\begin{enumerate}
\item a maximum at $\eta^*=\frac{C_{-1}}{C_1 + C_{-1}}$
\item symmetry defined by, $\forall \epsilon \in \left[0, 
    \frac{1}{C_1+C_{-1}}\right]$,
\begin{equation}
 \label{eq:CSsymmetryORIG}
 C^*\left(\eta^* - C_{-1} \epsilon \right) = 
 C^*\left(\eta^* + C_1 \epsilon \right),
\end{equation}
\end{enumerate}
\label{thm:Cprops}
\end{Thm}
\begin{proof}
Note that (\ref{eq:01minRisk}) can be written as
\begin{eqnarray}
  C^*_{C_1,C_{-1}} (\eta) = \left\{ \begin{array}{ll}
         C_{-1} (1-\eta), & \text{if $f^* \geq 0 $};\\
        C_1 \eta, & \text{if $f^*<0$},\end{array} \right. 
\end{eqnarray}
The two lines $C_{-1} (1-\eta)$ and $C_1 \eta$ intersect and form the maximum at $\eta=\frac{C_{-1}}{C_1+C_{-1}}$.

When $\epsilon=0$ we have the trivial case of 
$C^*\left(\frac{C_{-1}}{C_1+C_{-1}} \right)=C^*\left(\frac{C_{-1}}{C_1+C_{-1}} \right)$.

When $0<\epsilon \le \frac{1}{C_1+C_{-1}}$ we have $\eta=\frac{C_{-1}}{C_1+C_{-1}} -C_{-1}\epsilon < \frac{C_{-1}}{C_1+C_{-1}}$
in which case from (\ref{eq:Bayesnec}), $f^*<0$ and 
\begin{eqnarray}
C^*_{C_1,C_{-1}} (\eta)=C_1\eta=
C_1\left(\frac{C_{-1}}{C_1+C_{-1}}-C_{-1}\epsilon \right) =\frac{C_1C_{-1}}{C_1+C_{-1}}-C_1C_{-1}\epsilon
\end{eqnarray}

When $0<\epsilon \le \frac{1}{C_1+C_{-1}}$ we also have $\eta=\frac{C_{-1}}{C_1+C_{-1}} +C_1\epsilon > \frac{C_{-1}}{C_1+C_{-1}}$
in which case from (\ref{eq:Bayesnec}), $f^*>0$ and 
\begin{eqnarray}
C^*_{C_1,C_{-1}} (\eta)=C_{-1}(1-\eta)=
C_{-1}\left(1-\frac{C_{-1}}{C_1+C_{-1}}-C_1\epsilon \right) =\frac{C_1C_{-1}}{C_1+C_{-1}}-C_1C_{-1}\epsilon
\end{eqnarray}

Thus proving that 
\begin{eqnarray}
C^*_{C_1,C_{-1}} \left(\frac{C_{-1}}{C_1+C_{-1}} -C_{-1}\epsilon \right)=C^*_{C_1,C_{-1}} \left(\frac{C_{-1}}{C_1+C_{-1}} +C_1\epsilon \right)=\frac{C_1C_{-1}}{C_1+C_{-1}}-C_1C_{-1}\epsilon
\end{eqnarray}

\end{proof}

As noted by the following lemma, property 2. 
is in fact a generalization of property 1. 
\begin{lemma}
\label{Thm:LemmaDerivative}
Any concave function with the symmetry of (\ref{eq:CSsymmetryORIG}) 
also has property 1. of Theorem~\ref{Thm:RiskSymmetryProps}.
\end{lemma}
\begin{proof}
Taking the derivative of (\ref{eq:CSsymmetryORIG}) at $\epsilon=0$ leads to
\begin{eqnarray}
  C^{*'}\left(\frac{C_{-1}}{C_1+C_{-1}}  \right)(-C_{-1}) = 
   C^{*'}\left(\frac{C_{-1}}{C_1+C_{-1}}  \right)(C_1)
\end{eqnarray}
which is satisfied only when $C^{*'}\left(\frac{C_{-1}}{C_1+C_{-1}}\right)=0$. Given that $C^*$ is a concave function, $C^*$ is maximum at $\eta=\frac{C_{-1}}{C_2+C_{-1}}$.
\end{proof}

\subsection{Cost-sensitive Bayes consistent margin losses}

We extend the other losses used in machine learning to the cost-sensitive
paradigm by introducing the following set of margin loss function
\begin{eqnarray}
  L_{\phi,C_1,C_{-1}}(f,y) &=& \phi_{C_1,C_{-1}}(yf) \nonumber \\
    &=& \left\{
      \begin{array}{ll}
        \phi_1(f), & \text{if $y=1$} \\
        \phi_{-1}(-f), & \text{if $y=-1$}.
      \end{array}        \right.
    \label{eq:CSloss}
\end{eqnarray}
The associated conditional risk is
\begin{equation}
  C_{\phi,C_1,C_{-1}}(\eta,f) = \eta \phi_1(f) 
  + (1-\eta) \phi_{-1}(f)
 \label{eq:CondiRisk}  
\end{equation}
and is minimized by the predictor 
\begin{equation}
  \label{eq:f*C}
  f^*_{\phi,C_1,C_{-1}}(\eta) = \arg\min_{f}   C_{\phi,C_1,C_{-1}}(\eta,f).
\end{equation}
This leads to the minimum conditional risk
\begin{eqnarray}
  C^*_{\phi,C_1,C_{-1}}(\eta) &=& \!\!\!\! \eta \phi_1(f^*_{\phi,C_1,C_{-1}}(\eta)) \nonumber \\
  &+& \!\!\!\! (1-\eta) \phi_{-1}(-f^*_{\phi,C_1,C_{-1}}(\eta)).
  \label{eq:CSminRisk}
\end{eqnarray}


Similar to the cost insensitive case, our choice of $\phi_i(\cdot)$ in~(\ref{eq:CSloss}) cannot be arbitrary and we require certain properties for the loss function.
These desirable properties are addressed by extending the approach of~\Citet{savage-loss}.   

\begin{Thm}
\label{Thm:TheorySteps}
Let $g(\eta)$ be any invertible function, 
$J(\eta)$ any convex function, and $\phi_i(\cdot)$
determined by the following steps:
\begin{enumerate}
\item use (\ref{eq:I1Jprime}) and (\ref{eq:I2Jprime}) 
  to obtain the $I_{1}(\eta)$ and $I_{-1}(\eta)$,
  and let $C_{\phi,C_1,C_{-1}}(\eta, f)$ be defined by (\ref{eq:CondiRisk}).
\item 
  set $\phi_1(g(\eta))=-I_{1}(\eta)$ and
  $\phi_{-1}(-g(\eta))=-I_{-1}(\eta) $.
\end{enumerate}
Then $g(\eta) = f^*_{\phi,C_1,C_{-1}}(\eta)$ if and only if
$J(\eta) = -C^*_{\phi,C_1,C_{-1}}(\eta)$.
\end{Thm}
\begin{proof}
From 1. and Theorem~\ref{thm:savage}, it follows that
\begin{displaymath}
  \eta I_1({\hat \eta}) + (1-\eta) I_1({\hat \eta}) 
\end{displaymath}
has maximum value $J(\eta)$, when ${\hat \eta} = \eta$. 
From 2. the same holds for
\begin{displaymath}
  -\eta \phi_1(g({\hat \eta})) - (1-\eta) 
  \phi_{-1}(-g({\hat \eta}))
\end{displaymath}
and
\begin{displaymath}
  J(\eta) = 
  -\eta \phi_1(g(\eta)) - (1-\eta) \phi_{-1}(-g(\eta)).
\end{displaymath}
It follows from~(\ref{eq:CondiRisk})-(\ref{eq:CSminRisk}) that, 
$g(\eta) = f^*_{\phi,C_1,C_{-1}}(\eta)$ if and only if 
$J(\eta) = -C^*_{\phi,C_1,C_{-1}}(\eta)$.

\end{proof}

The theorem provides a generative method for designing the loss functions $\phi_i(\cdot)$ starting from any pair of invertible function $g(\eta)$ and convex function
$J(\eta)$. The resulting loss function will satisfy (\ref{eq:CondiRisk})-(\ref{eq:CSminRisk}), when $g(\eta) = 
f^*_{\phi,C_1,C_{-1}}(\eta)$ and $J(\eta) = -C^*_{\phi,C_1,C_{-1}}(\eta)$. 

What remains to be answered is how to choose $f^*_{\phi,C_1,C_{-1}}(\eta)$, and $C^*_{\phi,C_1,C_{-1}}(\eta)$ so as to ensure cost sensitive Bayes consistency. The following theorem provides a sufficient condition on $f^*_{\phi,C_1,C_{-1}}(\eta)$ for the Bayes optimality of the loss function.
 
\begin{Thm}
\label{Thm:TheoryLink}
Any invertible predictor $f(\eta)$ with symmetry
\begin{equation}
f^{-1}(-v) = \frac{2C_{-1}}{C_1+C_{-1}} - f^{-1}(v)
\label{eq:linksymCS}
\end{equation}
satisfies the necessary and sufficient conditions for cost-sensitive 
optimality of (\ref{eq:Bayesnec}) with $\gamma=\frac{C_{-1}}{C_1 + C_{-1}}$.
\end{Thm}
\begin{proof}
Assume that $f(\eta)=v$ is monotonically increasing. Note that $f^{-1}(0)=\frac{C_{-1}}{C_1+C_{-1}}$ which along with
$\eta=f^{-1}(v)$ leads to $f(\frac{C_{-1}}{C_1+C_{-1}})=0$. If $\eta>\frac{C_{-1}}{C_1+C_{-1}}$ then from (\ref{eq:linksymCS})
we have $f^{-1}(-v)<\frac{C_{-1}}{C_1+C_{-1}}$, applying (\ref{eq:linksymCS}) again it follows that  $f(\eta)>\frac{C_{-1}}{C_1+C_{-1}}$. Similarly, if $\eta<\frac{C_{-1}}{C_1+C_{-1}}$ then $f(\eta)<\frac{C_{-1}}{C_1+C_{-1}}$.
\end{proof}

In other words,  any predictor $f^*_{\phi,C_1,C_{-1}}(\eta)$ that satisfies 
(\ref{eq:linksymCS}) will be guaranteed to have a conditional risk that
is minimized by the cost-sensitive Bayes decision rule. 

What remains to be discussed is how to specify $C^*_{\phi,C_1,C_{-1}}(\eta)$ which will determine the risk of the optimal classifier. The goal is to approximate  the minimum conditional cost-sensitive zero-one risk (minimum cost sensitive Bayes risk) given in (\ref{eq:01minRisk}) as best as possible so as to achieve the minimum cost sensitive Bayes error. This is formally presented in the following theorem

\begin{Thm}
\label{Thm:ChooseCondRisk}
The minimum risk of any cost sensitive loss in the form of (\ref{eq:CSloss}) and derived from Theorem \ref{Thm:TheorySteps} can be made to be arbitrarily close, in the expectation, to the minimum cost sensitive Bayes error by choosing the minimum conditional risk of the loss to be arbitrarily close to the minimum conditional risk of the cost sensitive zero-one loss function.
\end{Thm}
\begin{proof}
\begin{eqnarray}
&&R^*_{\phi,C_1,C_{-1}}- \epsilon_{C_1,C_{-1}}= R^*_{\phi,C_1,C_{-1}} - R^*_{C_1,C_{-1}} = \\
&& E_X[C^*_{\phi,C_1,C_{-1}}] - E_X[C^*_{C_1,C_{-1}}] =
E_X[C^*_{\phi,C_1,C_{-1}} - C^*_{C_1,C_{-1}}]
\label{eq:ExpectationMeasure}
\end{eqnarray}
Where we have used Theorem \ref{Thm:MinBayesErr} for the first equality.
\end{proof} 

While Theorem \ref{Thm:ChooseCondRisk} says that the true measure for determining $C^*_{\phi,C_1,C_{-1}}$ is the expectation of (\ref{eq:ExpectationMeasure}), Theorem \ref{Thm:RiskSymmetryProps} suggests a simpler rule of thumb for selecting $C^*_{\phi,C_1,C_{-1}}$. Property 1. assigns the largest risk to the locations on the 
classification boundary and requiring this property for $C^*_{\phi,C_1,C_{-1}}$ would be vital. Also, enforcing Property 2. further guarantees that the optimal risk has the symmetry of the minimum cost-sensitive Bayes risk. 

\begin{definition}
A minimum risk $C^*_{\phi,C_1,C_{-1}}(\eta)$ is of
\begin{enumerate}
\item Type-I if it satisfies property 1. but not 2. of 
  Theorem~\ref{Thm:RiskSymmetryProps}.
\item Type-II if it satisfies both properties 1. and 2.
\end{enumerate} 
\end{definition}
Risks of type-II are generally closer approximations to the cost-sensitive
Bayes risk than those of type I. Although, strictly speaking the true measure is the expectation of (\ref{eq:ExpectationMeasure}).

The combination of Theorems~\ref{Thm:RiskSymmetryProps}-\ref{Thm:ChooseCondRisk} 
leads to
a generic procedure for the design of cost-sensitive
classification algorithms, consisting of the following steps
\begin{enumerate}
\item select a predictor $f^*_{\phi,C_1,C_{-1}}(\eta)$
  that satisfies~(\ref{eq:linksymCS}).
\item select a concave minimum conditional risk using the measure of (\ref{eq:ExpectationMeasure}) or, as a simpler rule of thumb alternative, select a concave minimum conditional risk
  $C^*_{\phi,C_1,C_{-1}}(\eta)$ of type-I or type-II,
  which reduces to $C^*_{\phi}(\eta)$ when $C_1 = C_{-1} = 1$.
\item use (\ref{eq:I1Jprime}) and (\ref{eq:I2Jprime}) with $J(\eta) = 
  -C^*_{\phi,C_1,C_{-1}}(\eta)$ to obtain $I_{1}(\eta)$ and $I_{-1}(\eta)$.
\item find $\phi_i(\cdot)$ so that 
  $I_{1}(\eta) = - \phi_1(f^*_{\phi,C_1,C_{-1}}(\eta))$ and
  $I_{-1}(\eta) = - \phi_{-1}(-f^*_{\phi,C_1,C_{-1}}(\eta))$.
\item derive an algorithm to minimize the conditional risk of
  (\ref{eq:CondiRisk}).
\end{enumerate}
We next illustrate the practical application of this framework
by showing that the cost-sensitive exponential loss 
of~\Citet{CSAdaBoostICML} can be derived from a minimal conditional risk of 
Type-I. 

\subsection{Cost-sensitive exponential loss}

We start by recalling that AdaBoost is based on the 
loss $\phi(yf) = \exp(-yf)$, for which it can be shown that
\begin{eqnarray}
  &&\!\!\!\!\!\!\!\!\!\! C^*_\phi(\eta) = \eta\sqrt\frac{1-\eta}{\eta} 
  + (1-\eta) \sqrt\frac{\eta}{1-\eta} \nonumber \\
  &&\!\!\!\!\!\!\!\!\!\! \,\,\,\,\, \text{and} \,\,\,\,\, 
  f^*_\phi = \frac{1}{2}\log\frac{\eta}{1-\eta}.
  \label{eq:C*CSAda}
\end{eqnarray}
A  natural cost-sensitive extension is 
$f^*_{\phi,C_1,C_{-1}}(\eta) = 
\frac{1}{C_1+C_{-1}}\log\frac{\eta C_1}{(1-\eta)C_{-1}}$,
which is easily shown to satisfy~(\ref{eq:linksymCS}).
Noting that $C^*_\phi(\eta) = \eta\exp(-f^*_\phi) +
(1-\eta) \exp(f^*_\phi)$, suggests the cost-sensitive extension
\begin{eqnarray}
  C^*_{\phi,C_1,C_{-1}}(\eta) &=&\!\!\!\!
 \eta \left(\frac{\eta C_1}{(1-\eta)C_{-1}}\right)^{\frac{-C_1}{C_1+C_{-1}}} + \nonumber \\
 &&\!\!\!\!\!\!\!\!\!\!\!\! (1-\eta)\left(\frac{\eta C_1}{(1-\eta)C_{-1})}\right)^{\frac{C_{-1}}{C_1+C_{-1}}}.
\end{eqnarray}
This does not have the symmetry of~(\ref{eq:CSsymmetryORIG}) but 
satisfies property 1. of Theorem~\ref{thm:Cprops}. Hence, it
is a Type-I risk. It is also equivalent to~(\ref{eq:C*CSAda}) 
when $C_1 = C_{-1} = 1$. Finally, steps 1. and 2. of 
Theorem~\ref{Thm:TheorySteps} produce the loss
\begin{equation}
  \phi_{C_1,C_{-1}}(yf) = \left\{
      \begin{array}{ll}
        \exp(-C_1f), & \text{if $y=1$} \\
        \exp(C_{-1}f), & \text{if $y=-1$}
      \end{array}        \right.
\end{equation}
proposed in~\Citet{CSAdaBoostICML}. The resulting cost-sensitive boosting 
algorithm currently holds the best performance in the literature.

\section{Cost sensitive SVM}
Next we extend the hinge loss used in SVMs using the cost sensitive framework established in the previous section.
The cost sensitive SVM optimization problem is also derived.

\begin{figure*}[t]
  \centering
    \begin{tabular}{cc}
     \includegraphics[trim = 0.2in 2in 0.2in 2in, width=3in]{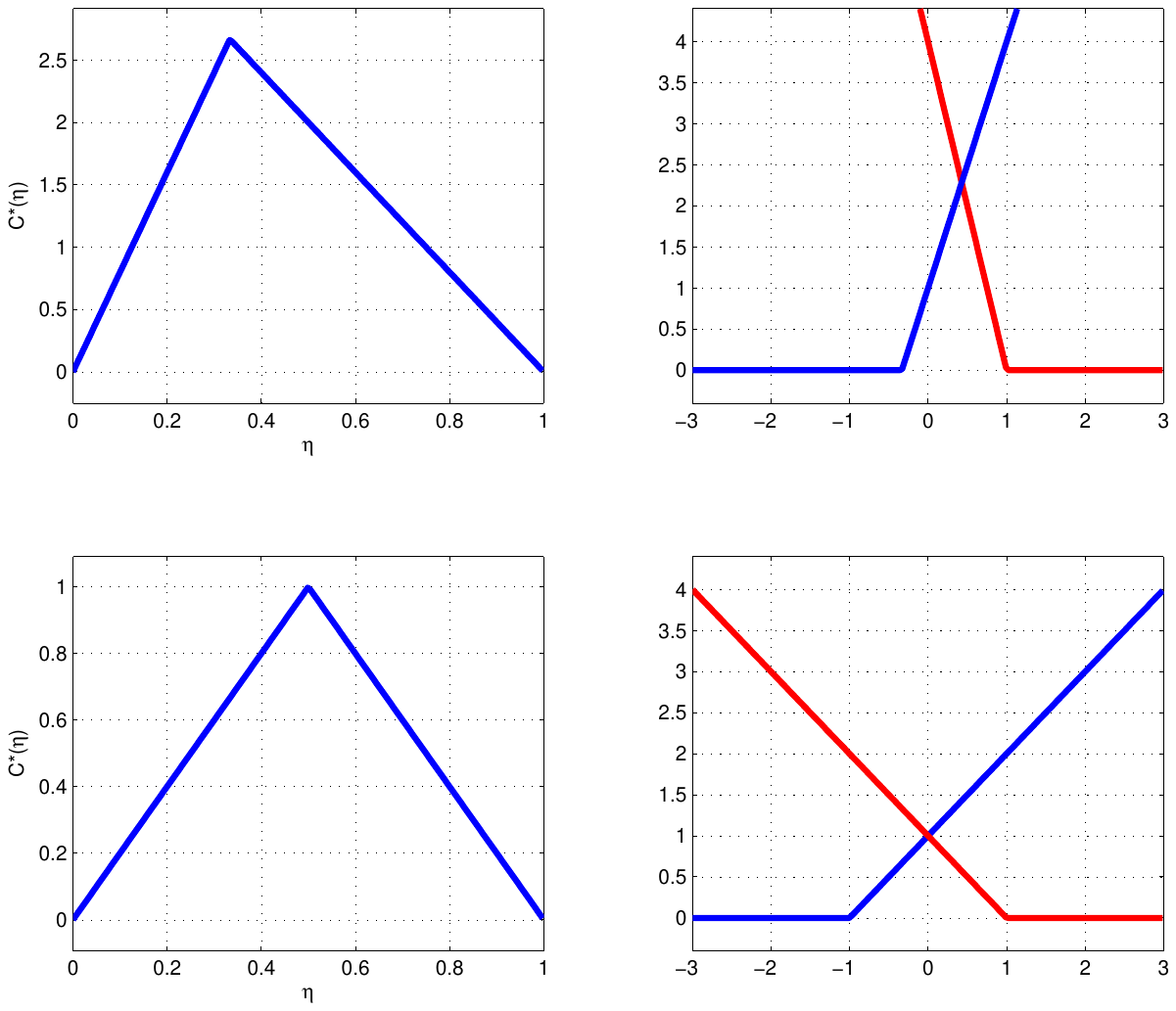}
    \includegraphics[trim = 0.2in 2in 0.2in 2in, width=3in]{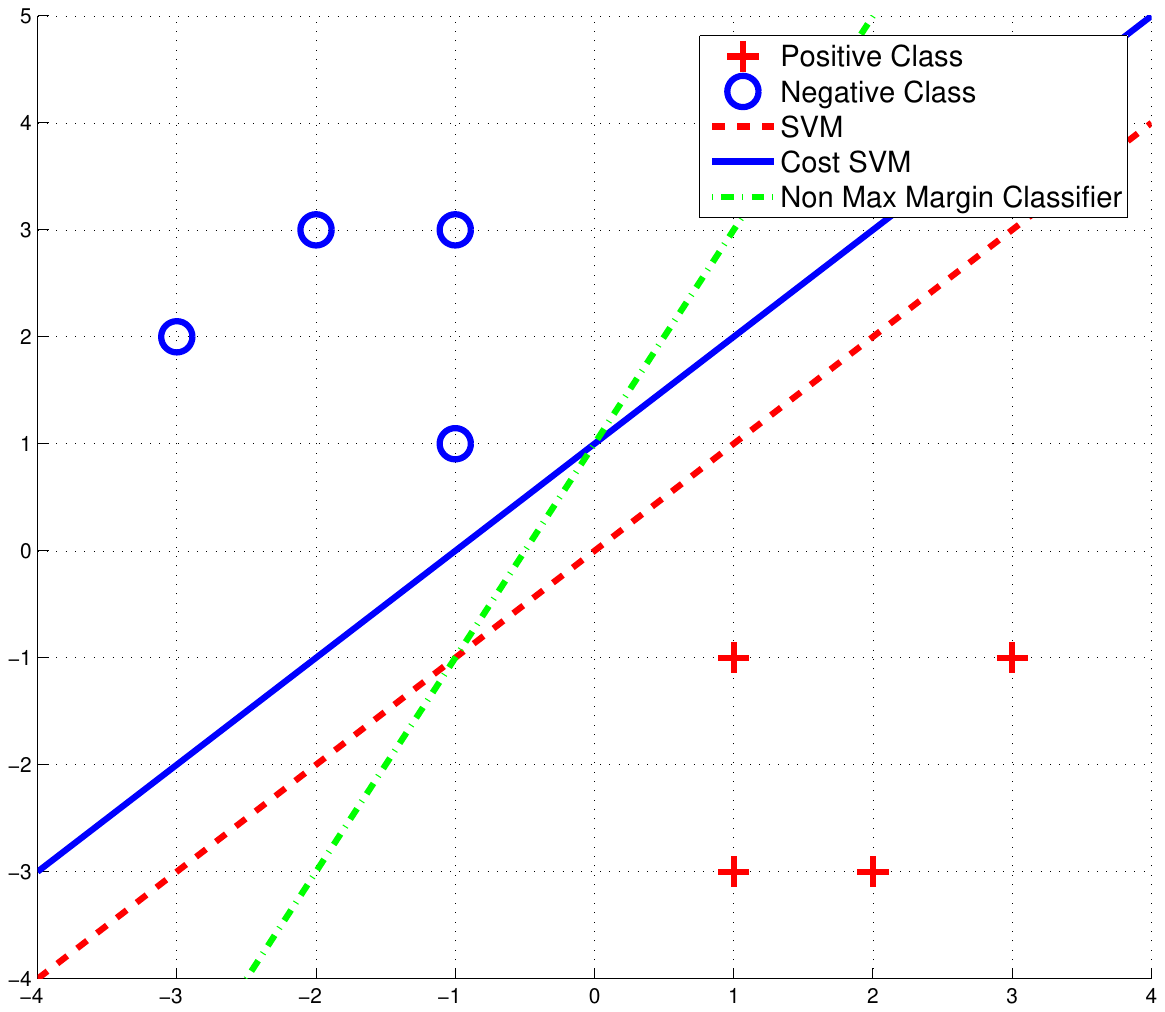} 
 \end{tabular}
\caption{  Left: concave $C^*_{\phi,C_1,C_{-1}}(\eta)$ function and 
corresponding  cost sensitive SVM loss function, top: $C_1=4$, $C_{-1}=2$,
bottom: $C_1=C_{-1}=1$. Right: linearly separable cost sensitive SVM.  }
\label{fig:CSSVMLossJP_FIXFINAL}
\end{figure*}

The SVM minimizes the risk of the
hinge loss $\phi(yf)= \lfloor 1-yf \rfloor_+$, where $\lfloor x \rfloor_+ = 
\max (x, 0)$. The associated risk is minimized by~\Citet{Zhang04} 
\begin{equation}
  \label{eq:fsvm}
  f^*_\phi(\eta) = sign(2\eta - 1)
\end{equation}
resulting in the minimum conditional risk
\begin{eqnarray*}
 &&\!\!\!\!\!\!\!\!\!\! C^*_\phi(\eta) =  1 - |2\eta-1| \\
 &&\!\!\!\!\!\!\!\!\!\! =  \eta  \lfloor 1- sign(2\eta-1)  \rfloor_+  + 
  (1-\eta)  \lfloor 1+ sign(2\eta-1)  \rfloor_+.
\end{eqnarray*}

We follow the generic procedure and replace the optimal cost-insensitive predictor
by its cost-sensitive counterpart
\begin{equation}
  f^*_{\phi,C_1,C_{-1}}(\eta)=sign( (C_1+C_{-1})\eta-C_{-1} ).
\end{equation}
which can be directly shown to satisfy~(\ref{eq:Bayesnec}).
This suggests choosing the cost-sensitive minimum conditional risk 
\begin{eqnarray}
  C^*_{\phi,C_1,C_{-1}}(\eta)= &&
  \eta \lfloor e-d \cdot sign((C_1+C_{-1})\eta-C_{-1})\rfloor_+ +   \\
&& (1-\eta) \lfloor b+a \cdot sign((C_1+C_{-1})\eta-C_{-1})\rfloor_+,   \nonumber
\end{eqnarray}  
which can be shown to satisfy~(\ref{eq:CSsymmetryORIG}) if and only if
\begin{eqnarray}
d \geq e  \hspace{0.5in} a \geq b  \hspace{0.25in} \text{and}
\hspace{0.25in} \frac{C_{-1}}{C_1}=\frac{a+b}{d+e}.
\label{eq:constraints}
\end{eqnarray} 
The hinge loss minimum conditional risk  satisfies the conditions of a Type-II loss function and is also a close approximation of the zero-one minimum conditional risk under the criteria of  Theorem -\ref{Thm:ChooseCondRisk}.

After steps 1. and 2. of Theorem~\ref{Thm:TheorySteps},
\begin{equation}
  \phi_{C_1,C_{-1}}(yf) = \left\{
      \begin{array}{ll}
        \lfloor e - d f \rfloor_+, & \text{if $y=1$} \\
        \lfloor b + a f \rfloor_+, & \text{if $y=-1$}.
      \end{array}        \right.
   \label{eq:CSlossSVMmm}
\end{equation}
This loss has four degrees of freedom, which 
control the margin and slope of the hinge components associated with the two
classes: positive examples are classified with margin $\frac{e}{d}$ and 
hinge loss slope $d$, while for negative examples the margin is 
$\frac{b}{a}$ and  slope $a$.

\subsection{Cost-sensitive SVM learning}
\label{sec:CSSVMLearning}
We consider the case where errors in the positive class are weighted more 
heavily, leading to the inequalities $\frac{b}{a} \le \frac{e}{d} $ and 
$d \ge a$. Choosing $e=d=C_1$ normalizes the margin of positive examples
to unity $(\frac{e}{d}=1)$. Selecting $b=1$ then fixes the scale 
of the negative component of the hinge loss, leading to $a=2C_{-1}-1$. 
The resulting cost sensitive SVM loss function is 
\begin{eqnarray}  \label{eq:CSH}
\phi_{C_1,C_{-1}}(yf)= 1_{\{y=1\}} C_1\lfloor1-yf\rfloor_+ + 1_{\{y=-1\}} \lfloor1-(2C_{-1}-1)yf\rfloor_+
\end{eqnarray}  
and the cost sensitive SVM minimal conditional risk is  
\begin{eqnarray}
  \label{eq:CstarSVMmm}
 &&\!\!\!\!\!\!\!\!\!\! C^*_{\phi,C_1,C_{-1}}(\eta)= \\
 &&\!\!\!\!\!\!\!\!\!\! \eta \lfloor C_1-C_1 \cdot sign((C_1+C_{-1})\eta-C_{-1})\rfloor_+ + \nonumber \\
  &&\!\!\!\!\!\!\!\!\!\! (1-\eta) \lfloor 1+(2C_{-1}-1) \cdot 
  sign((C_1+C_{-1})\eta-C_{-1})\rfloor_+ \nonumber
\end{eqnarray}  
with $C_{-1}\geq 1$ and $C_1 \geq 2C_{-1}-1$, so as to 
satisfy~(\ref{eq:constraints}). Figure~\ref{fig:CSSVMLossJP_FIXFINAL} 
presents plots of~(\ref{eq:CstarSVMmm})
and~(\ref{eq:CSH}), for both $C_1=4$, $C_{-1}=2$ and the cost 
insensitive case of $C_1=1$, $C_{-1}=1$ (standard SVM). Note that,
for the cost-sensitive SVM, the positive class has a unit margin,
while the negative class has a smaller margin of $\frac{1}{3}$. Also,
the slope of the positive component of the loss is $4$ while the negative 
component has a smaller slope of $3$. In this way, the loss assigns a 
higher cost to errors in the positive class when the data is not 
separable, while enforcing a larger margin for positive examples when the 
data is separable.
Replacing the standard hinge loss with~(\ref{eq:CSlossSVMmm}) in the standard SVM risk~\Citet{SVMTutorial} 
\begin{eqnarray}
\underset{w,b}{\text{argmin}}\sum_{\{i|y_i = 1\}}\!\!\!\!\!\!\ \lfloor C_1-C_1 (w^Tx_i+b)\rfloor_+  + \!\!\!\!\!\!\ \sum_{\{i|y_i = -1\}}\!\!\!\!\!\!\ \lfloor 1+(2C_{-1}-1) (w^Tx_i+b)\rfloor_+ + \frac{1}{2C} ||w||^2,
\end{eqnarray}
leads to the primal problem 
\begin{equation}\label{eq:csprimal}
\begin{aligned}
&\underset{w,b,\xi_i}{\text{argmin}} \ &&\frac{1}{2} ||w||^2 + C \left[  C_1 \sum_{\{i|y_i=1 \}} \xi_i   + \frac{1}{\kappa} \sum_{\{i|y_i=-1 \}} \xi_i \right]  \\
&\text{subject to} &&(w^Tx_i+b) \geq 1-\xi_i ; \hspace{0.2in} \ \ \  y_i=1 \\
& &&(w^Tx_i+b) \le - \kappa +\xi_i ; \hspace{0.2in} y_i=-1 
\end{aligned}
\end{equation}
with  
\begin{equation}
  \kappa = \frac{1}{2C_{-1}-1}, \hspace{0.5in} 0<\kappa\le 1 \le \frac{1}{\kappa} \le C_1.
  \label{eq:svmpars}
\end{equation}
This is a quadratic programming problem similar to that of the 
standard cost-insensitive SVM with soft margin weight parameter $C$.
In this case, cost-sensitivity is controlled by the parameters
$C_1, \frac{1}{\kappa},$ and $\kappa$. The parameter $\kappa$ is responsible 
for cost-sensitivity in the separable case. Under the constraints 
$C_{-1} \geq 1$, $C_1 \geq 2C_{-1} - 1$, ($0<\kappa\le 1 \le \frac{1}{\kappa} \le C_1$), of a type-II risk, it imposes a smaller margin on 
negative examples. On the other hand, $C_1$ and $\frac{1}{\kappa}$ control the 
relative weights of margin violations, assigning more weight to positive 
violations. This allows control of cost-sensitivity when the data
is not separable.

Obviously, this primal problem could be defined through heuristic
arguments. However, it would be difficult to justify precise choices for the parameters of~(\ref{eq:svmpars}). Furthermore, the derivation
above guarantees that the optimal classifier implements the Bayes
decision rule of~(\ref{eq:Bayesnec}) with $\gamma = \frac{C_{-1}}{C_1
+C_{-1}}$, and its risk is a type-II approximation to the cost-sensitive
Bayes risk. No such guarantees would be possible for an heuristic solution.

To obtain some intuition about the cost-sensitive extension, we consider 
the synthetic problem of Figure \ref{fig:CSSVMLossJP_FIXFINAL}, where the two 
classes are linearly separable. The figure shows three separating lines.
The green line is an arbitrary separating line that does not maximize the
margin.  The red line is the standard SVM solution, which has maximum 
margin and is equally distant from the nearest examples of the two classes. 
The blue line is the solution of~(\ref{eq:csprimal}) for $C_1=4$ and
$C_{-1}=2$ (the $C$ parameter is irrelevant  when the data is separable). It is also a maximum margin solution, but trades-off the distance 
to positive and negative examples so as to enforce a larger positive
margin, as specified. Overall, an increase in $C_{-1}$ (decrease in $\kappa$) guarantees a larger 
positive margin. For a given $C_{-1}$, increasing $C_1$ (so that 
$C_1 \ge 2C_{-1}-1$) increases the cost of errors on positive examples, 
enabling control of the miss rate when the classes are not separable. 

We note that for the separable case, a limited level of cost sensitive performance can be achieved using the BP-SVM formulation of (\ref{eq:BPSVM}) along with a small weight parameter $C$ ($C<\frac{1}{2}$), but a small $C$ is undesirable in general as it leads to an under trained model 
with training errors even when the data is separable. The CS-SVM formulation, on the other hand, provides a maximum margin solution regardless of the chosen weight parameter $C$ . The CS-SVM is preferable even in the inseparable case because  increasing the weight parameter $C$, in an attempt to reduce training error, inevitably leads to over training in the BP-SVM formulation. This is not necessarily the case for the CS-SVM formulation which allows a decrease of the margin of the negative samples (through an appropriate choice of $\kappa$) and a relative increase in the margin of the positive samples,  independent of the weight parameter $C$ and does not lead to over training. In other words, unlike the BP-SVM formulation, the CS-SVM does not simply over train on the positive class, it  maximizes the margin on this class. This can also be seen, with added clarity,  in the dual CS-SVM formulation which is discussed in the next section.

\section{Cost-sensitive SVM in the dual}
The dual and kernelized  formulation of the CS-SVM of (\ref{eq:csprimal}) can be derived  as 
\begin{equation}\label{eq:csdual0}
\begin{aligned}
&\underset{\alpha}{\text{argmax}} \hspace{0.1in}  
&& \sum_i \alpha_i \left(\frac{y_i+1}{2} - \frac{\kappa(y_i-1)}{2}  \right) -\frac{1}{2} \sum_i \sum_j \alpha_i \alpha_j y_i y_j K(x_i,x_j) \\
& \text{subject to} && \sum_i \alpha_i y_i=0  \\
& &&      0 \le \alpha_i \le C C_1 ; \hspace{0.3in}  y_i=1 \\         
& &&      0 \le \alpha_i \le \frac{C}{\kappa} ;  \hspace{0.3in} \ \ \; y_i=-1 
\end{aligned}
\end{equation}
which reduces to the standard SVM dual when $C_1 = C_{-1} = 1$.
Unlike the previous BM-SVM and BP-SVM algorithms, the CS-SVM algorithm performs regardless of the separability of the data and the chosen slack penalty $C$.
This can be further studied in detail by writing the dual problem ~(\ref{eq:csdual0}) as

\begin{equation}\label{eq:csdual1}
\begin{aligned}
&\underset{\alpha}{\text{argmax}} \hspace{0.1in} && \sum_i \alpha^+_i + \kappa\sum_i \alpha^-_i   -\frac{1}{2} \sum_i \sum_j \alpha_i \alpha_j y_i y_j K(x_i,x_j)  \\
& \text{subject to} && \sum_i \alpha_i y_i=0  \\
&&&      0 \le \alpha^+_i \le C C_1  \\         
&&&      0 \le \alpha^-_i \le \frac{C}{\kappa} 
\end{aligned}
\end{equation}
with 
\begin{align} 
&0 < \kappa \le 1 \le \frac{1}{\kappa} \le C_{1} \label{eq:lossconst} \\
&{\alpha_i}^{+}=\{{\alpha}_i | y_i=1\},\ \ {\alpha_i}^{-}=\{{\alpha}_i | y_i=-1\}. \nonumber
\end{align}
Moreover, since  $\alpha_i \ge 0$ and $\kappa=1-(1-\kappa)$ we can rewrite  (\ref{eq:csdual1}) with an \lone norm term as 
\begin{equation}\label{eq:csdual2}
\begin{aligned}
&\underset{\alpha}{\text{argmax}} \hspace{0.1in} &&   -\frac{1}{2} \alpha^T Y KY\alpha + \textbf{1}^T\alpha   -(1-\kappa) {\|{\alpha}^{-}\|}_1 \\
&\text{subject to} &&  \alpha^T y=0  \\
&&&     0 \preceq \alpha^+ \preceq C C_{1}  \\         
&&&      0 \preceq \alpha^- \preceq \frac{C}{\kappa} . 
\end{aligned}
\end{equation}
where $Y=Diag(y)$ and $\textbf{1}$ is the vector of all ones.

When $C_{1}=1$ and $\kappa=1$, i.e. $C_{-1}=1$, the problem of (\ref{eq:csdual2}) reverts to the standard SVM dual formulation. This implies that (\ref{eq:csdual2}) is totally compatible with standard dual solvers and its implementation  on existing SVM dual solvers is a non-issue.

If we transform problem (\ref{eq:csdual2}) into a minimization problem, the term ${\|{\alpha}^-\|}_1$ acts as an \lone regularization term with positive coefficient $(1-\kappa)$. Another difference with the standard cost insensitive SVM (CI-SVM) and BP-SVM dual problem is that in  \eqref{eq:csdual2}, the upper bounds on $\alpha^+$ and $\alpha^-$ are scaled differently. 
In particular, because $\frac{1}{\kappa} \le C_{1}$, the active upper bound constraints on $\alpha^+_i$ are relaxed, compared to $\alpha^-_i$.
In summary, the CS-SVM dual problem (\ref{eq:csdual2}) has two major differences compared to the CI-SVM dual problem:
\begin{enumerate}
\item  \lone regularization on ${\alpha}^-$.
\item relaxed  inequality constraints on ${\alpha}^+$. 
\end{enumerate}
These modifications have nontrivial consequences which connect regularization theory and sensitivity analysis to cost-sensitive learning. We study the implications of these modifications by first representing the CI-SVM dual problem  as a regularized risk minimization problem which allows us to explain the extra regularization term $-(1-\kappa) {\|{\alpha}^-\|}_1$  for both the case of cost sensitive learning and imbalanced learning problems. Subsequently, we study the affect of relaxing the inequality constraint on ${\alpha}^+$ using sensitivity analysis.

\begin{figure*}[t]
  \centering
    \begin{tabular}{cccc}
\includegraphics[trim = 0.2in 1.75in 0.2in 1.75in, width=1.4in]{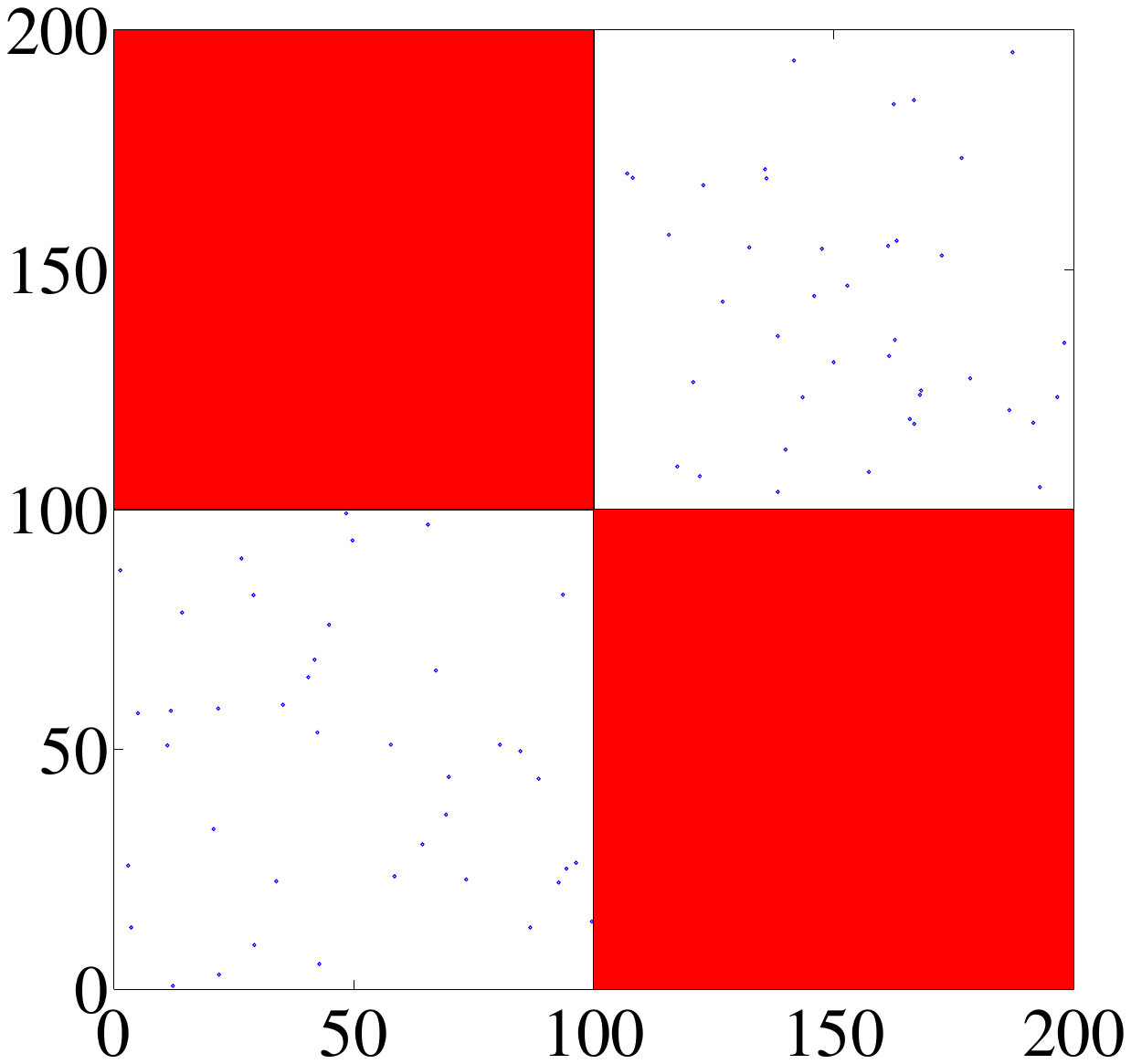} & \includegraphics[trim = 0.2in 1.75in 0.2in 1.75in, width=1.4in]{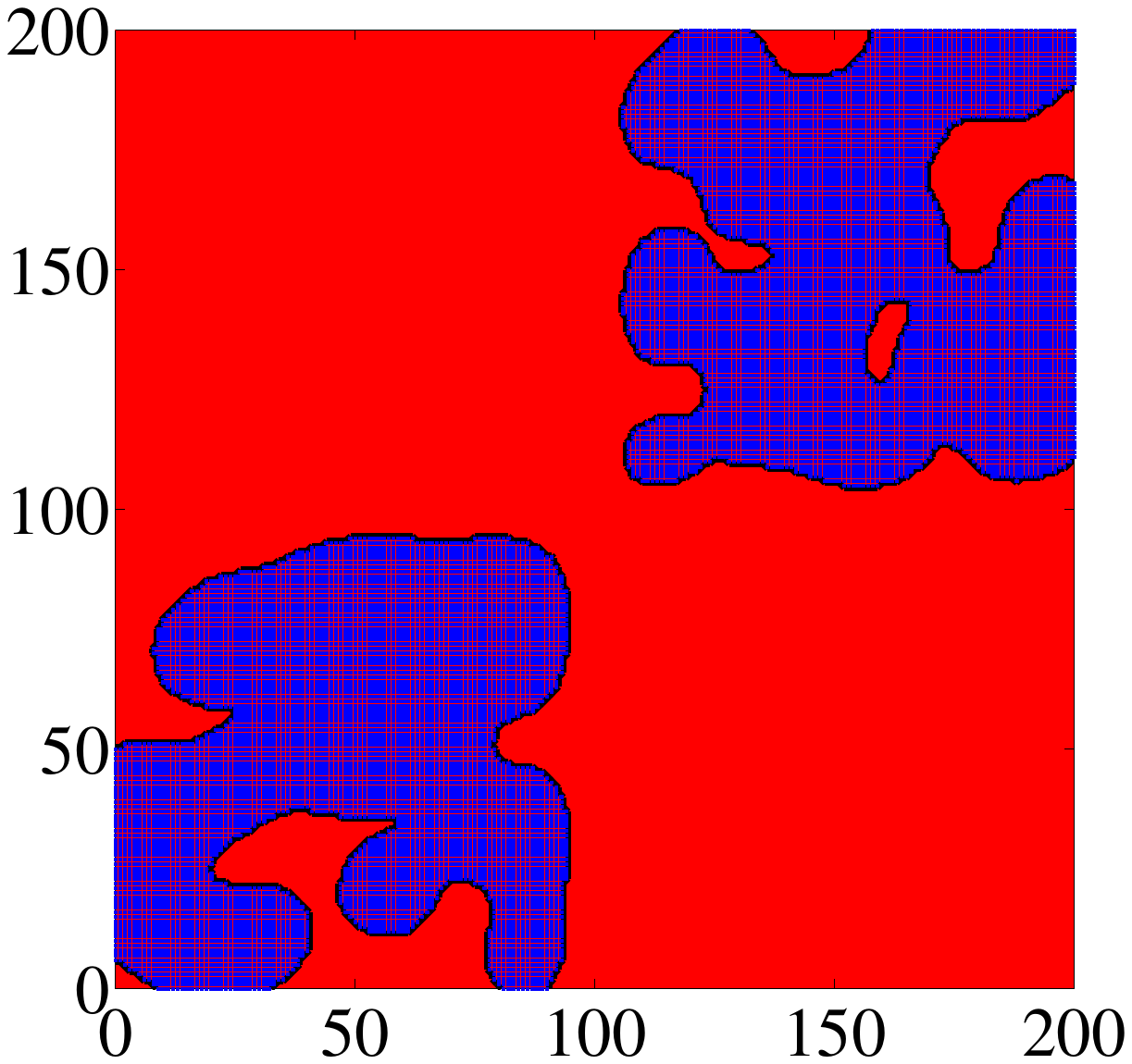} & \includegraphics[trim = 0.2in 1.75in 0.2in 1.75in, width=1.4in]{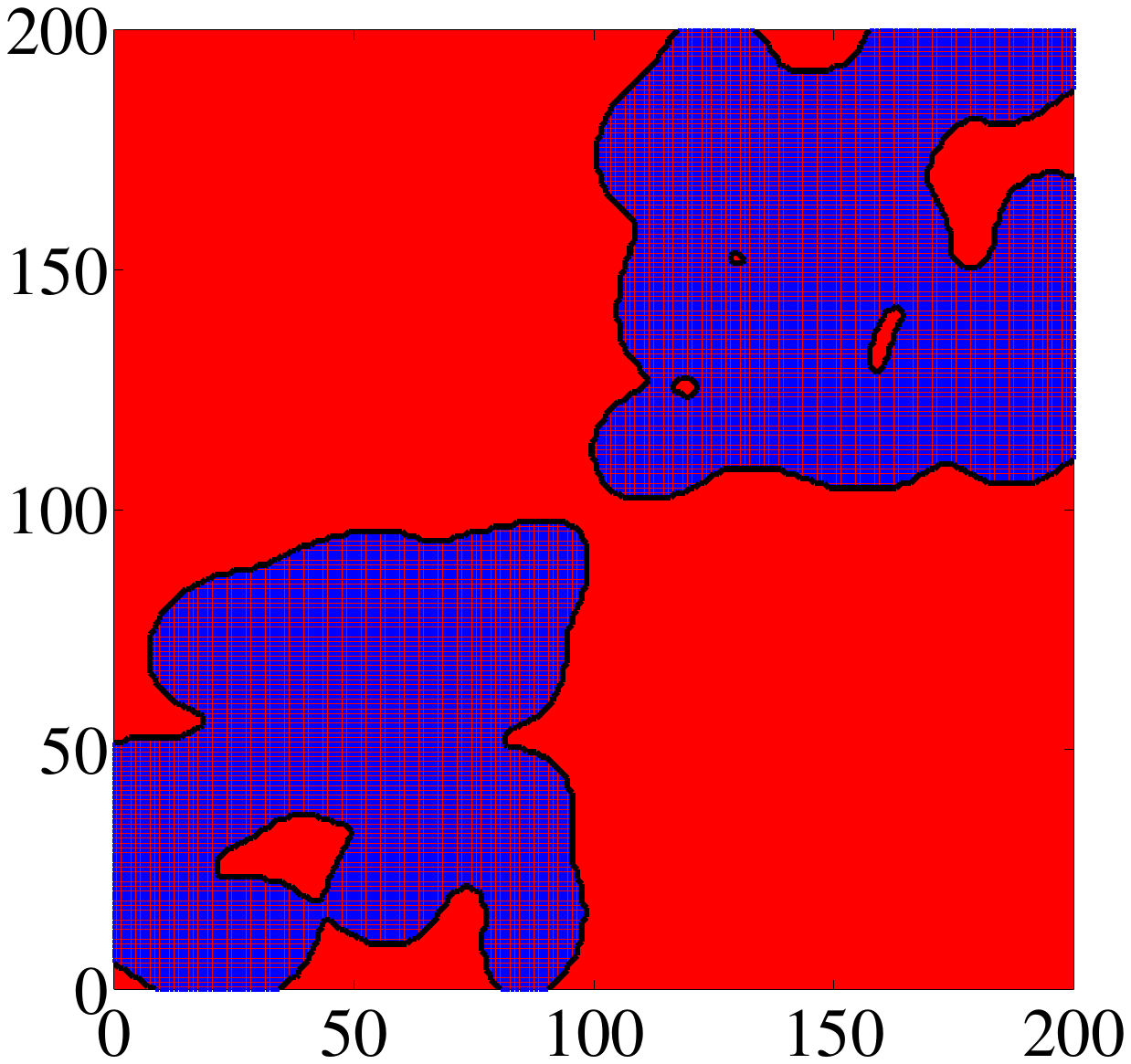} & \includegraphics[trim = 0.2in 1.75in 0.2in 1.75in, width=1.4in]{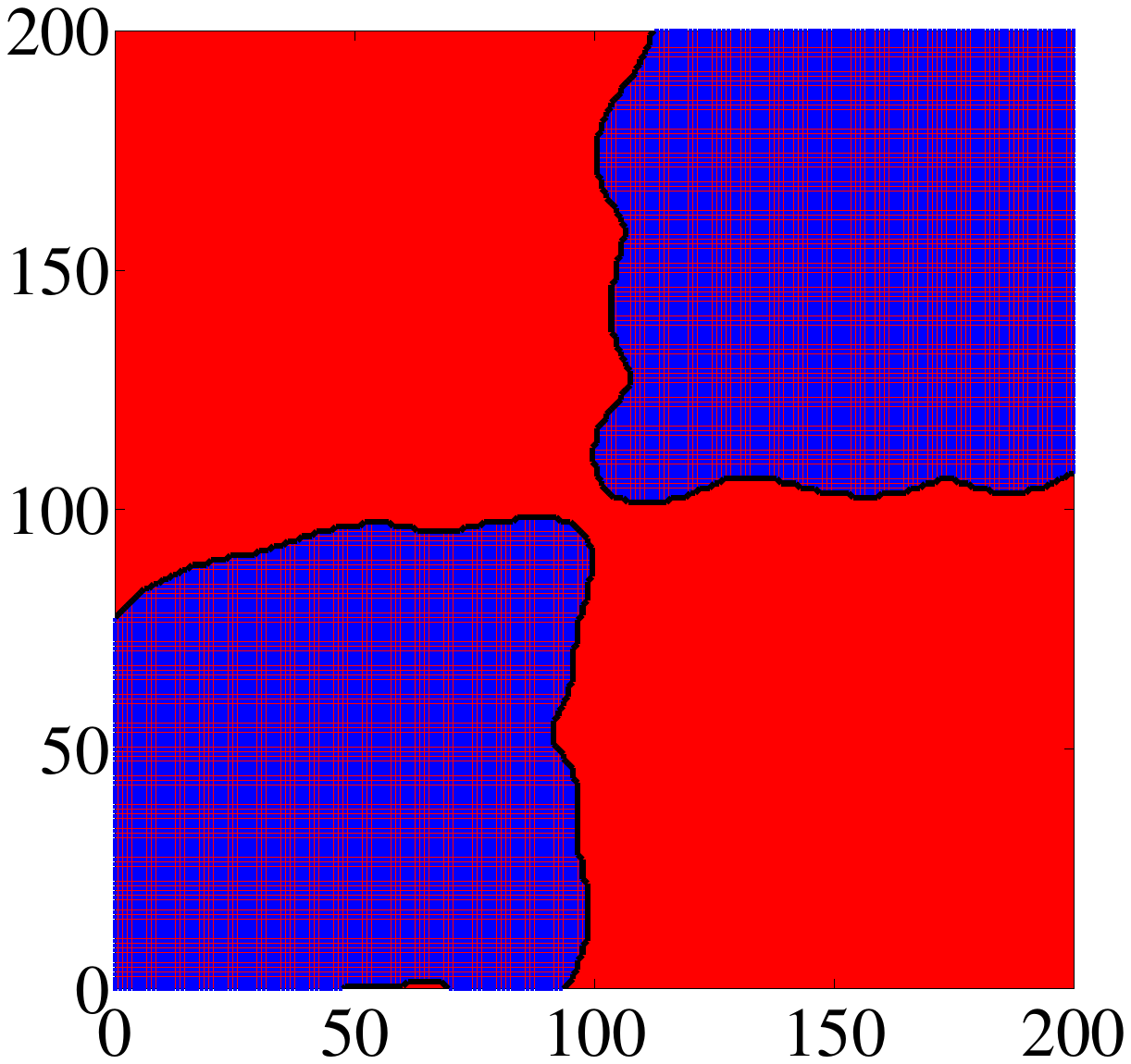} \\
             (a) & (b) &(c) & (d) \\
     \includegraphics[trim = 0.2in 1.75in 0.2in 1.75in, width=1.4in]{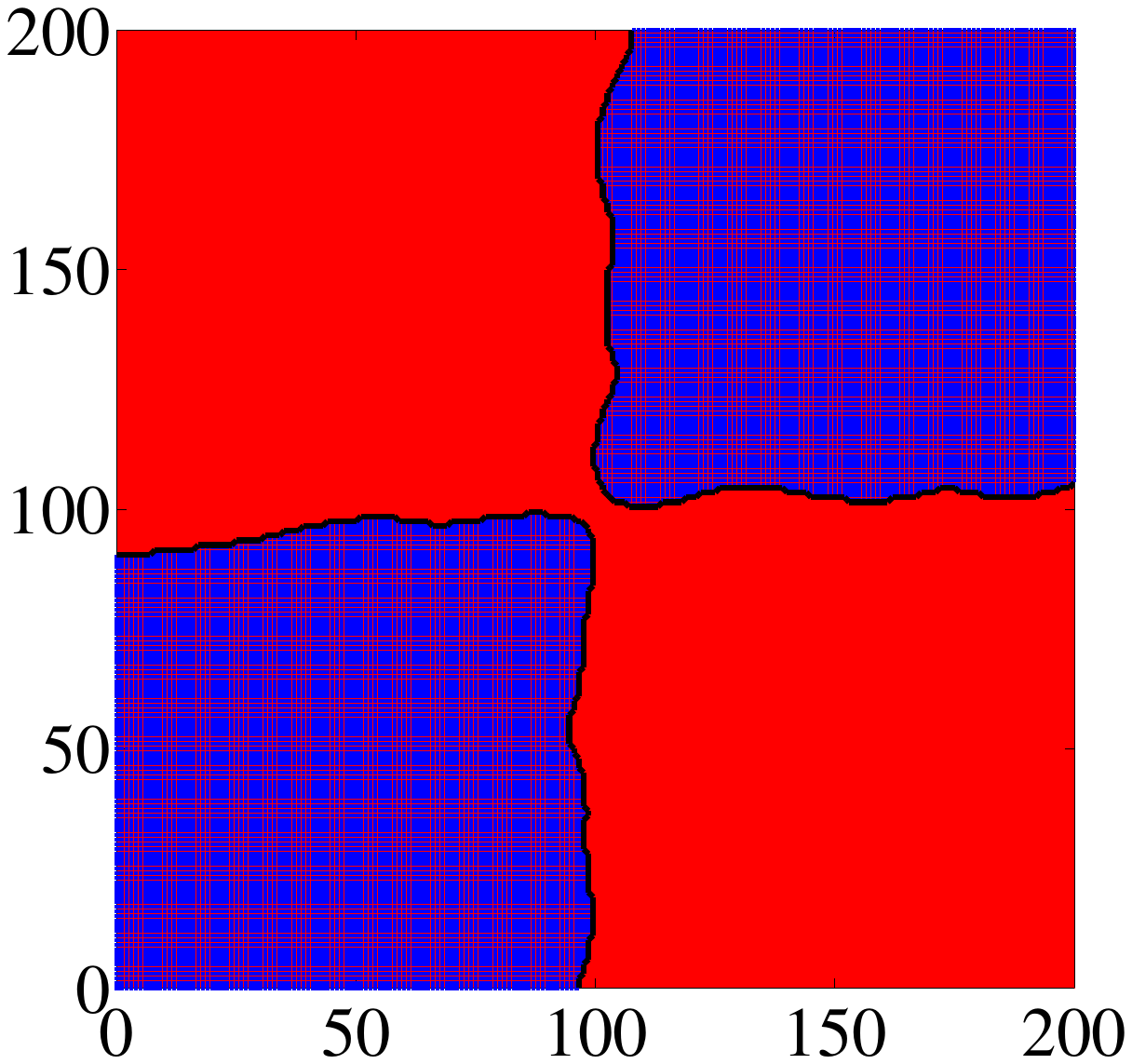} & \includegraphics[trim = 0.2in 1.75in 0.2in 1.75in, width=1.4in]{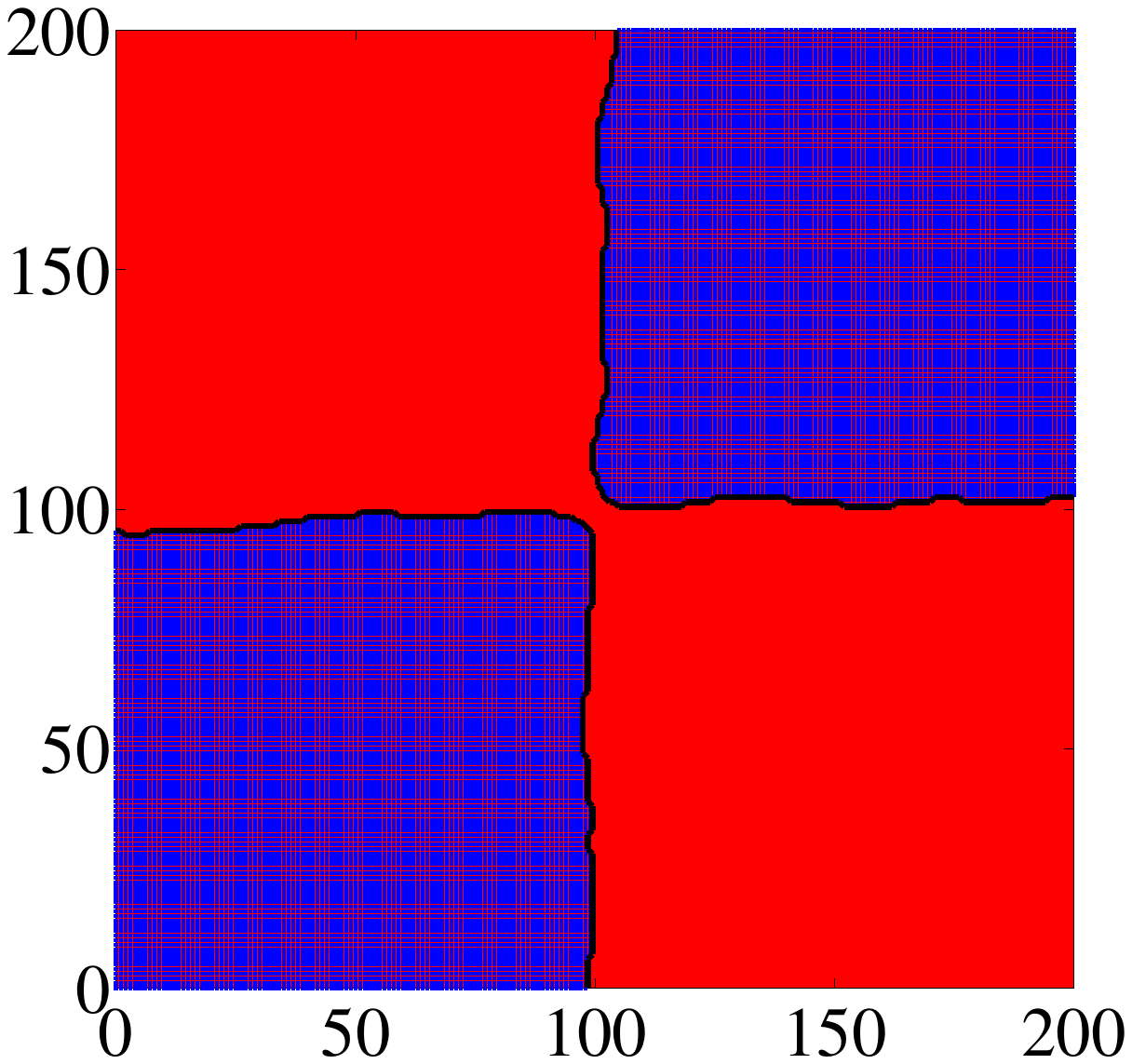} & \includegraphics[trim = 0.2in 1.75in 0.2in 1.75in, width=1.4in]{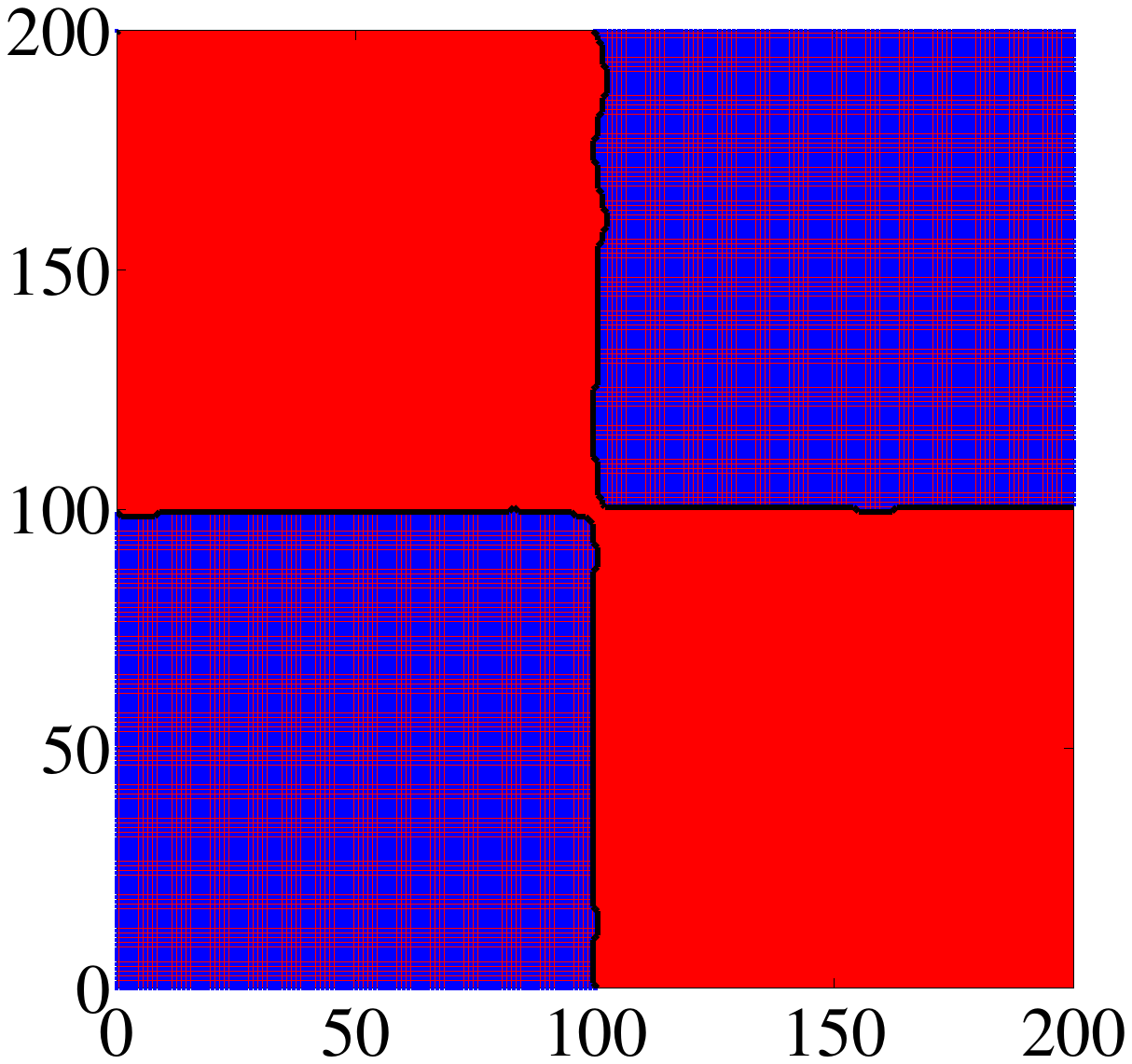} & \includegraphics[trim = 0.2in 1.75in 0.2in 1.75in, width=1.4in]{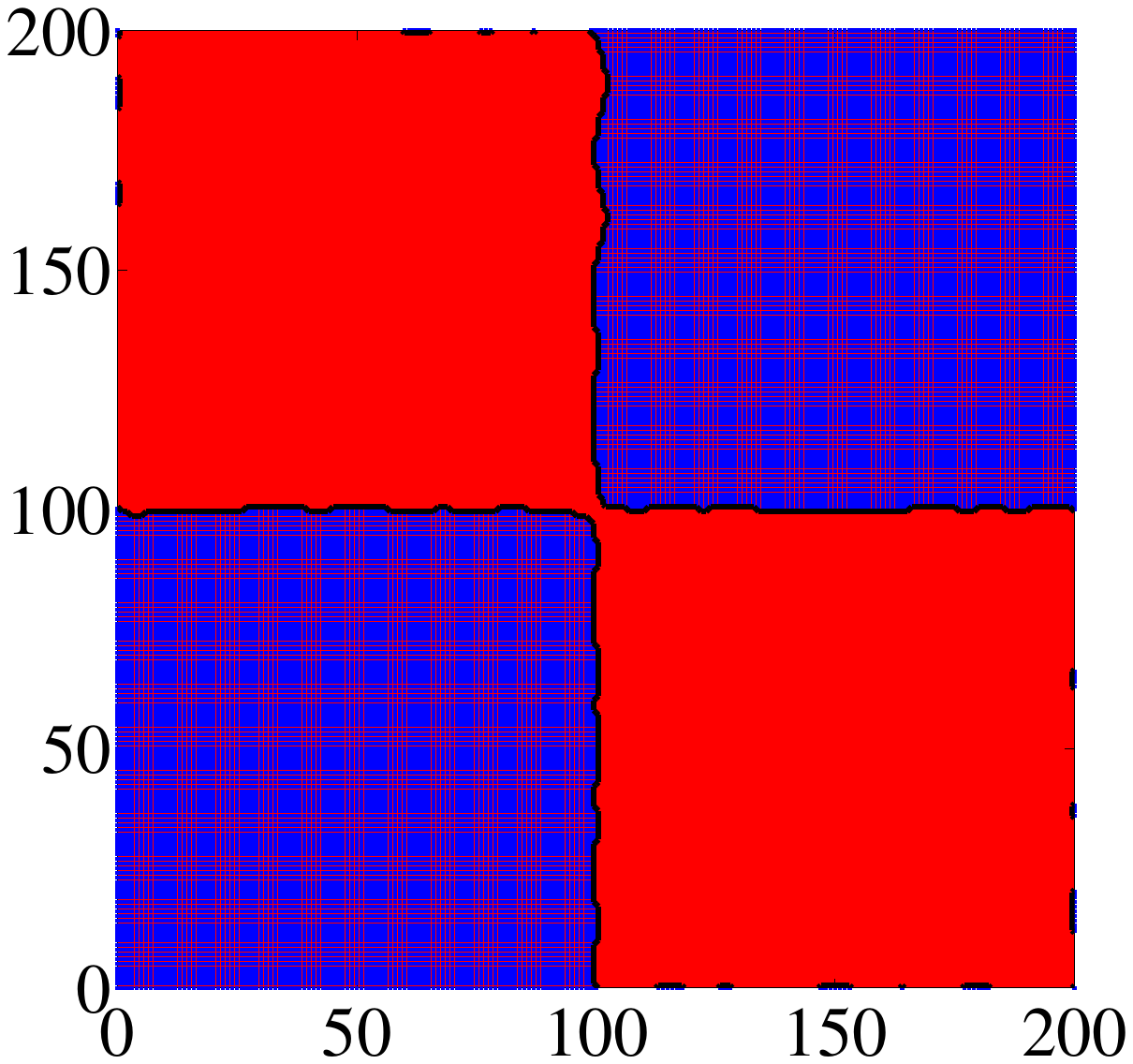} \\
     (e) & (f) &(g) & (h) \\
 \end{tabular}
\caption{ (a) The Checkerboard dataset with imbalance ratio 1:1000, (b) classification result of BM-SVM (c) classification result of BP-SVM with $C_1=100$ and $C_{-1}=1$ (CS-SVM with $\kappa=1$), (d) classification result of CS-SVM with $\kappa=0.5$, (e) classification result of CS-SVM with $\kappa=0.25$, (f) classification result of CS-SVM with $\kappa=0.1$, (g) classification result of CS-SVM with $\kappa=0.01$, (h) classification result of CS-SVM with $\kappa=0.001$.} 
\label{fig:checker}
\end{figure*}

\subsection{Regularization on Lagrange multipliers}

In this subsection we study the effects of \lone regularization on  $\alpha^-$ in the dual problem, while considering imbalanced dataset learning and cost-sensitive learning separately.

\subsubsection{Imbalanced dataset learning}
 In many applications examples from the target (positive) class are outnumbered by the non-target class. Moreover, in  multi-class classification problems where the number of classes are large and a one-versus-all scheme is used, the number of examples in each individual class is usually small compared to the rest of the examples, leading to a highly imbalances problem. These sorts of imbalances occur with different intensity, with ratios between the minority and majority class ranging from  1:10 to 1:$10^6$ ~\Citet{provost01}. 

For the SVM training problem, the  number of support vectors  grows linearly with the number of examples ~\Citet{steinwart04}, and this implies that the number of support vectors for each class grows linearly with the number of examples of that class. Therefore, the same imbalance, if not worse, happens in the number of nonzeros in of the solution. In other words, when the dual problem is solved, most of the support vectors belong to the majority class. The problem becomes more apparent when we take into account the equality constraint of \eqref{eq:csdual2}

\begin{eqnarray}
\sum_i \alpha_i y_i=0,  
\end{eqnarray}
which implies
\begin{eqnarray}
\|\alpha^+\|_1 = {\|\alpha^-\|}_1 
 \end{eqnarray}
Also, results of ~\Citet{steinwart04} implies that for imbalanced datasets
\begin{eqnarray}
 \znorm{\alpha^+}\ll\znorm{\alpha^-}.\label{eq:card}
\end{eqnarray}
This results in an irregular solution, with the $\alpha^+_i$s taking values close to the upper bound $C$  and the $\alpha^-_i$s taking values close to the lower bound zero. ~\Citet{SVM-KBA1} illustrated this problem by conducting an experiment on a 2D Checkerboard dataset with different imbalance ratios as seen in Figure \ref{fig:checker}(a) . They showed that in the case of imbalanced data, the decision boundary is unwillingly shifted toward the minority class. This is because of a lack of enough examples (support vectors) for the minority class that reside close to the correct decision boundary. When enough examples  don't exist at the right place, the margin relies on other examples farther away from the ideal decision boundary, resulting in the decision boundary shifting toward the minority class. They also equivalently illustrated that this is caused by irregular values in the dual variables. This problem persists in the BM-SVM and BP-SVM formulation as a result of their flawed implementation of the asymmetric margin, and can be seen in Figure \ref{fig:checker} which show the classification results for the BM-SVM, BP-SVM and CS-SVM on the Checkerboard dataset.

\begin{figure*}[t]
  \centering
	 \begin{tabular}{cc}
     \includegraphics[width=2.8in]{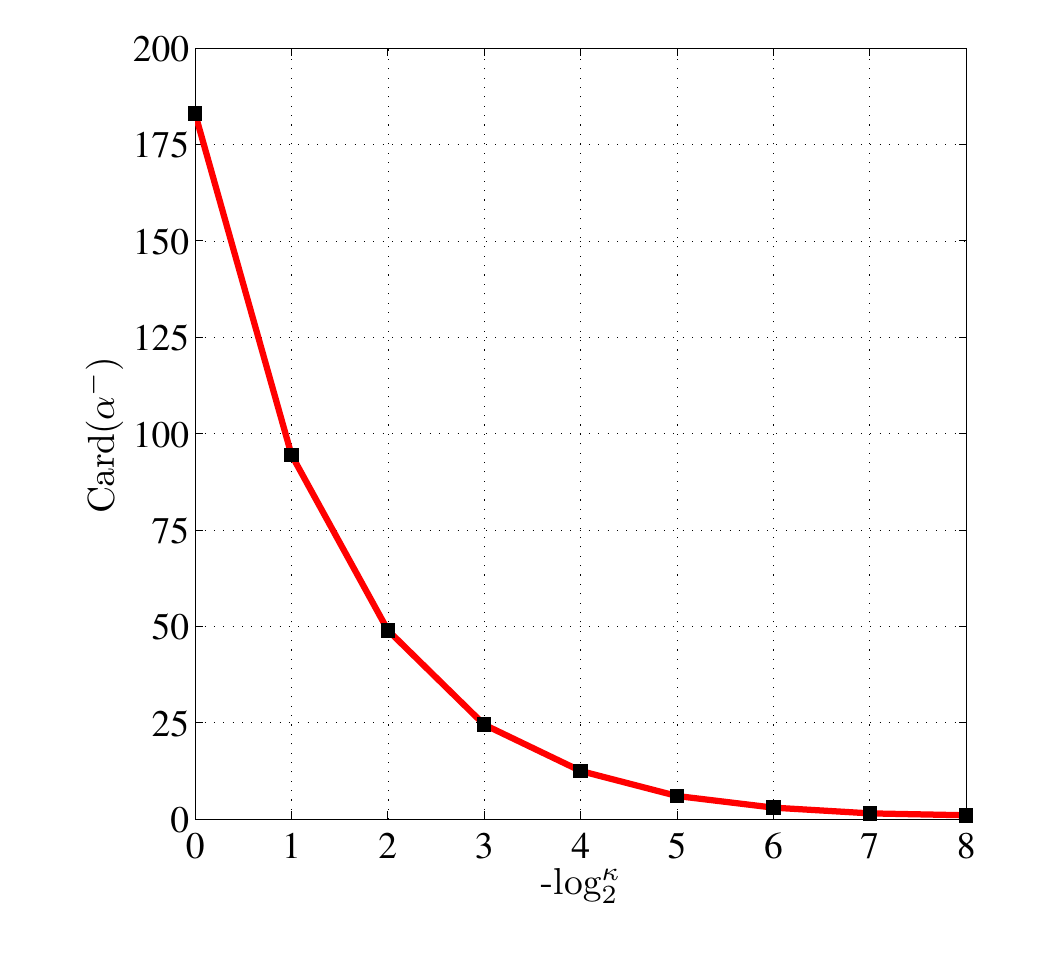} 
    &\includegraphics[width=2.8in]{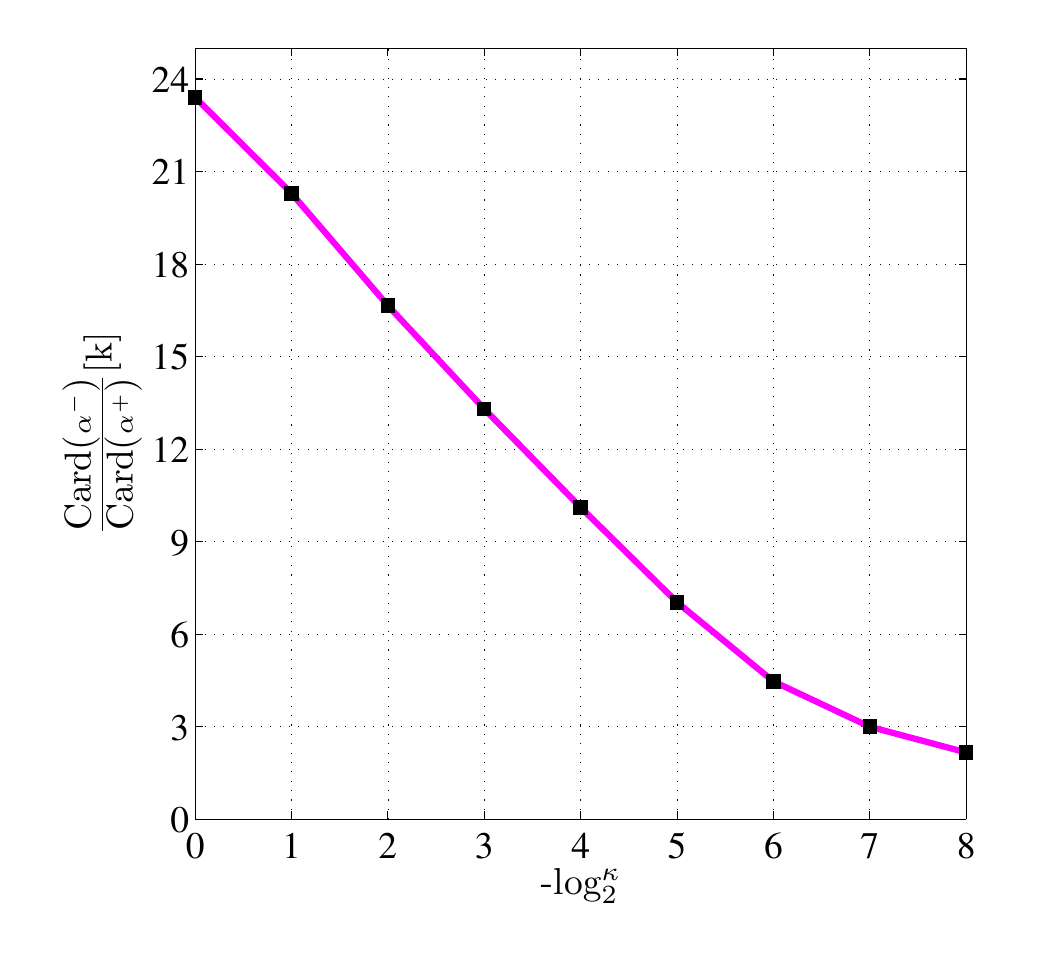} \\
        (a) & (b)  \\
	 \end{tabular}
\caption{ The CS-SVM algorithm for different choices of $\kappa$ is applied to
the covertype UCI dataset which is imbalanced with a ratio of 1:211. Starting at $\kappa=1$, CS-SVM acts as the BP-SVM. (a) shows the reduction in the number of $\alpha^-$ as $\kappa$ decreases and  (b) shows the reduction in the imbalance ratio as $\kappa$ decreases. As $\kappa$ decreases the number of negative support vectors is reduced so that by $\kappa=2^{-256}$ the imbalance ratio between support vectors approaches  $1$. }
\label{fig:reg}
\end{figure*}

Given that for imbalanced dataset problems the vector $\alpha^-$ has small non sparse elements while the vector $\alpha^+$ is highly sparse \eqref{eq:card}, the natural remedy  is to regularize the non-sparse part of the solution, $\alpha^-$, with a sparsity inducing \lone regularizer  ~\Citet{book:convex}. This leads to a sparse $\alpha^-$,  at the solution which is now both balanced and regularized.
The CS-SVM problem (\ref{eq:csdual2}) uses the same technique to deal with the problem of imbalanced datasets by choosing appropriate choice of $\kappa$. As $\kappa$ tends to zero the regularization coefficient $(1-\kappa)$ increases resulting in an increased regularization of the $\alpha^-_i$s, which enforces larger margin for minority (positive) class. Figure \ref{fig:checker}(g) shows that for a highly imbalanced checkerboard data, an small $\kappa=0.01$ corrects the decision boundary, close to the optimal one. 
Choosing $\kappa<0.01$  violates the condition \eqref{eq:lossconst} and has a diminishing return, i.e., leads to preferring the majority class as shown in Figure \ref{fig:checker}(h).
 
Also, Figure \ref{fig:reg} illustrates the effect of the CS-SVM regularization on the number of support vectors of each class in the solution. The CS-SVM algorithm with different choices of $\kappa$ is applied to the covertype UCI dataset which is imbalanced with a ratio of 1:211, which as the regularization coefficient $(1-\kappa)$ increases, $\alpha^-$  becomes sparser (Figure \ref{fig:reg}(a)). This leads to an equivalence between the number of non-zero components of $\alpha^-$ and $\alpha^+$ (Figure \ref{fig:reg}(b)).

Therefore, the CS-SVM  in the dual, applies a sparsity inducing \lone regularization on the $\alpha^-$ and when dealing with imbalanced datasets, the CS-SVM implicitly prevents unwanted movement of the discriminant boundary toward the minority class by enforcing margin to be asymmetric.

\subsubsection{Cost-sensitive learning} 
As shown in the previous section, regularization of any class results in a smaller margin for that class. So, in the cost-sensitive learning setting which costs are known, CS-SVM reduces the margin for the class with the lower cost, or equivalently increases the margin for the class with the higher cost.

In general, the extra \lone regularization in the CS-SVM dual problem makes the margin asymmetric, in favor of the minority class or the class with higher cost for imbalanced data learning and cost-sensitive learning, respectively.

\begin{figure*}[t]
\xymatrixrowsep{1in}
\xymatrixcolsep{2in}
\xymatrix{
\!\!\!\!\!\!\!\!
\textbf{Primal} \ \  \text{\small{Max. Margin Sep.}}(w,\Psi, \phi) 
\ar@{<->}[d]_{Dual}
\ar@{<->}[r]^{w^T\Psi_i \leftrightarrow \beta^TK_i}_{\Psi^T\Psi \leftrightarrow K} & \text{\small{Reg. Risk Min.}}(\beta,K,\phi) \ar@{<->}[d]^{Dual}\\
\textbf{Dual}  \ \  \text{\small{Dual of: Max. Margin Sep.}}(\alpha,K, \phi^*) 
\ar@{<->}[r]_{K  \leftrightarrow K^{-1}}^{\alpha \leftrightarrow YK^{-1}z}& \text{\small{Reg. Risk Min.}}(z,K^{-1},\phi^*)}
\caption{ The commutative diagram for existing SVM formulations essentially depends on associated parameter spaces $w,\alpha,\beta,z$ and feature spaces $\Psi,K,K^{-1}$. The matrix $\Psi^T$ is the Cholesky factor of $K$, i.e. $K=\Psi^T\Psi$, with its $i^{th}$ row corresponding to the feature space representation of the example $x_i$, i.e., $\Psi_i=\varphi(x_i)$. }
\label{fig:comm}
\end{figure*}
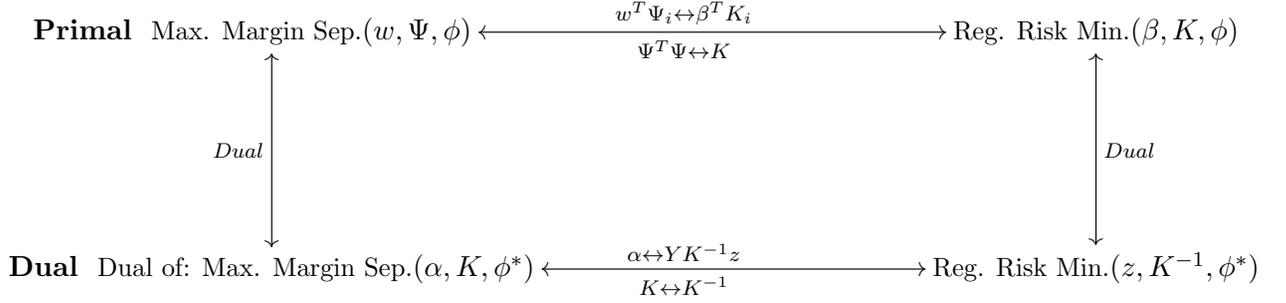

\subsection{Regularization on basis expansion coefficients}
In the previous section we showed how  the Lagrange dual~\cite{book:convex} of the CS-SVM performed \lone regularization on the support vectors. Rather, in this section we show that the Fenchel dual~\Citet{rockafellar70} of the CS-SVM performs \lone regularization on the basis coefficients of the discriminant function. A general regularization problem~\Citet{Tikhonov77} for given dataset $\Dc$, loss function $\Lc$, trade-off hyperparameter $C$, regularizer $\Omega$ and Hilbert space $\Hc$ can be written as 
\begin{eqnarray}
\underset{f \in \Hc}{\text{argmin}} \hspace{0.1in}  \Omega(f) +  \Lc(f;\Dc,C)  \label{eq:tikhonov}
\end{eqnarray} 
which by representer theorem \Citet{Scholkopf2001Learning}, \eqref{eq:regriskSVM} has a minimizer of form of 
\begin{eqnarray}
f({\bf x}) = \sum_{x_i \in \Dc}\beta_iK({\bf x},x_i)+b \label{eq:falpha}.
\end{eqnarray}
which for Hing loss $\phi$, the primal problem becomes \citet{Chapelle-primal}: 
\begin{eqnarray}
\underset{\beta,b}{\text{argmin}} \hspace{0.1in}  \frac{1}{2}\beta^TK\beta +  \sum_i{\phi(y_i ( \beta^TK_i+b))} \label{eq:regriskSVM}
\end{eqnarray}
where $K_i$ is the $i^{th}$ column of the kernel matrix. As shown in the Appendix \ref{app:fenchel}, the Fenchel dual problem of \eqref{eq:regriskSVM} can be written as 
\begin{equation}\label{eq:regriskdual0}
\begin{aligned}
\underset{z}{\text{argmax}}  &&-\Omega^*(g) -  \sum_i{\phi^*(y_ig_i)} 
\end{aligned}
\end{equation} 
which $z \in \Rbb^n$ is dual variable, and $g=K^{-1}z$ is the dual decision function. Figure \ref{fig:comm} depicts the relationship between problem \eqref{eq:regriskdual1} for the existing SVM formulations.

As shown in the Appendix \ref{app:fenchel}, the CS-SVM dual problem can be written as a regularized risk minimization problem
\beq\label{eq:regriskdual5}
\underset{z}{\text{argmax}}  &&-\Omega^*(g) -  \sum_i \phibp^*(y_ig_i) - (1-\kappa) \onenorm{g^-}
\eeq
which $g^-$ is a vector of $g_i$s which $y_{i}=-1 $. 

Also, by substituting $\Omega^*, \ \phibp^*$ (see Appendix  \ref{app:fenchel}) and setting \footnote{Note that $f(x_i)=K_i^T\beta$ and $g_i=g(x_i)=z^TK^{-1}_i$ are primal and dual decision functions.} $g=K^{-1}z$ we have
\beq\label{eq:regriskdual1} 
&\underset{z}{\text{argmax}}  &&- \frac{1}{2} z^TK^{-1}z +  y^Tg - (1-\kappa)\onenorm{g^-} \\
&\text{subject to }&& \onenorm{ 
g^+} =\onenorm{g^-}\\
 &&&  0\preceq g^+ \preceq CC_1 \\
 &&&  0\preceq -g^- \preceq \frac{C}{\kappa}
\eeq
There are several points to make:
\begin{itemize}
	\item By setting $YK^{-1}z=Yg=\alpha$, we can retrieve the SVM's dual problem \eqref{eq:csdual2}, which by using the fact that $\beta=Y\alpha$ \cite{Chapelle-primal}, we have $g=\beta$ in \eqref{eq:falpha}, \eqref{eq:regriskSVM} and \eqref{eq:regriskdual1}. This reveals an interesting duality property: $f=z$ and $g=\beta$, i.e. primal variable is equal to dual decision function and vice versa.
	\item The \lone regularization term and equality constraint in \eqref{eq:regriskdual1} can be regarded w.r.t either basis expansion coeeficients $\beta$ or dual decision values $g$.
\end{itemize}

Compared to CI-SVM and BP-SVM dual problems, problem \eqref{eq:regriskdual1} performs an extra \lone regularization on the basis expansion coefficients which has a different interpretation in imbalanced data learning and cost-sensitive learning:
\paragraph{Imbalanced Data Learning} 
The quantities $\znorm{\beta^+}$ and $\znorm{\beta^-}$ reflect the number of basis functions of each class which contribute the decision function, which similar to their $\alpha$-counter parts they are highly imbalanced, i.e., $\znorm{\beta^+} \ll \znorm{\beta^-}$.
This means that the discriminant function is mostly made up of data-dependent kernel bases of the majority class.
which leads to over train on the majority class while under training the minority class. 
Similar to basis pursuit ~\Citet{chen01}, CS-SVM adds a \lone regularization on the basis expansion coefficients to alleviate the problem of over-training on the majority class by balancing the number of basis functions contributing to the decision function.

\paragraph{Cost-sensitive learning} 
In the cost-sensitive learning setting, errors of misclassifying one class is higher than the other class and we can translate this to the learning algorithm by choosing more basis functions of the target class. This idea can be implemented by performing  \lone regularization on the expansion coefficients of the lower-cost class ($\beta^-$) \eqref{eq:regriskdual1}.

\section{Example-dependent cost-sensitive learning}
In many applications such as computational advertising \cite{Agarwal2011}, medical diagnosis \cite{inductive-cost}, information retrieval \cite{CS-IR}, fraud detection \cite{fraud1,fraud2} and  business decision-making \cite{ED-03} the cost of misclassifying  an individual example differs from other examples including those of the same class. This gives rise to the concept of example-dependent cost-sensitive (ED-CS) learning. 

There main methods to ED-CS learning is \emph{direct cost-sensitive method} \cite{CP-unknown}, which considers a threshold for each example according to its costs, i.e. 
\beq \label{eq:metacost}
h(x_i,C_{1},C_{-1})=
\begin{cases} 
 1,  & \frac{\eta}{1-\eta} \ge \frac{C_{-1}}{C_{1}}\\
 -1,& \text{otherwise}
\end{cases}
\eeq
\emph{MetaCost} \cite{metacost} changes the labels of training set according to \eqref{eq:metacost}, and then trains with the new labels.
\Citet{ED-03} and  \Citet{ED-SVM} proposed methods where the training examples are resampled according to the example cost probability distribution of the data.
Despite their simplicity, resampling methods  may suffer from over fitting caused by duplicate examples.
More recently, ~\Citet{scott11} proposed, but did not to implemented, an example-based version of BP-SVM loss function which we call ED-BP-Hinge. The ED-BP-Hinge loss is defined for each example with label $y$, decision value $f$ and cost $c$ as

\begin{eqnarray}  \label{eq:EDBPH}
\phi(y,f,c)= c\lfloor1-yf\rfloor_+
\end{eqnarray}  

In dealing with the example dependent cost sensitive learning problem we extend the CS-SVM loss of (\ref{eq:CSlossSVMmm})  to the ED-CS-Hinge defined as 
\begin{eqnarray}  \label{eq:EDCSH}
\phi(y,t,c)=
\begin{cases} 
 c\lfloor 1-y t\rfloor_+ ,  & \mathrm{for} \ \ y=+1,\\
 \lfloor 1-(2c-1)yt\rfloor_+ ,& \mathrm{for}\ \ y=-1.
\end{cases}
\end{eqnarray}  
the ED-CS-Hinge loss function inherits the benefits of the CS-SVM loss including the added flexibility of choosing an asymmetric margin of the loss  when compared to the ED-BP-Hinge. In the experimental study we implement an example dependent cost sensitive SVM based on the ED-CS-Hinge loss and show an improvement over the ED-BP-Hinge based SVM and other SVM based algorithms on the KDD98 dataset.

\section{Performance measure}
The evaluation of cost sensitive algorithms requires a flexible performance measure that can incorporate different costs and priors. We adopt the cost sensitive zero-one risk which can be written as 
\begin{eqnarray}
\label{eq:RCS} 
&R_{CS} &=   E_{Y,{\bf X}} [L_{C_1,C_{-1}}(f({\bf x}),y)|{\bf X} = {\bf x}]  \nonumber \\          
&&      = \sum_{y} \sum_{{\bf x}} P_{{\bf X | Y}}({\bf X} = {\bf x}|Y=y)P_{\bf Y}(y) L_{C_1,C_{-1}}(f({\bf x}),y) \nonumber\\
&&      = \sum_{y} P_{\bf Y}(+1) \sum_{{\bf x}} P_{{\bf X|Y}}({\bf X} = {\bf x}|Y=+1) L_{C_1,C_{-1}}(f({\bf x}),+1)  \nonumber\\
&&   + \sum_{y} P_{\bf Y}(-1) \sum_{{\bf x}} P_{{\bf X|Y}}({\bf X} = {\bf x}|Y=-1) L_{C_1,C_{-1}}(f({\bf x}),-1)  \nonumber\\
&&	= P_{1}C_1P_{FN} +P_{-1}C_{-1}P_{FP}
\end{eqnarray}
where $P_1$ and $P_{-1}$ are the class priors and $P_{FN}$ and $P_{FP}$ are the false negative and false positive rates respectively.
This performance measure readily simplifies to the well known probability of error measure $R_{CI}=P_1P_{FN} + P_{-1}P_{FP}$, which we call cost insensitive risk.

Finding the best cost sensitive zero-one risk of (\ref{eq:RCS}) can be as an instance of vector optimization problem. Each classifier produces a set of vectors $(P_{FP},P_{FN}) $ which should be compared w.r.t. in nonnegative orthant, i.e., $\Rbb^2_+$) which induces component wise inequality in $\Rbb_+^2$. The minimal elements of this set comprise the Pareto optimal frontier \cite{book:convex} which is also known as the ROC curve in detection theory. Different points on the ROC of a classifier can be found by the vector scalarization optimization problem of
\begin{eqnarray}
\label{eq:scale} 
&&    \min_{P_{FP},P_{FN}}\   \lambda_1P_{FP}+\lambda_2P_{FN}
\end{eqnarray}
Choosing $(\lambda_1,\lambda_2)=(P_1C_1,P_{-1}C_{-1})$ results in the following optimization problem 
\begin{eqnarray}
\label{eq:scaleRisk}
&&    \min_{P_{FP},P_{FN}}\   P_{1}C_{1}P_{FP}+P_{-1}C_{-1}P_{FN}.  
\end{eqnarray}
which has an objective function equal to the cost sensitive zero-one risk of (\ref{eq:RCS}). This means that by using the cost sensitive zero-one risk as the performance measure and choosing a certain $(P_1C_1 , P_{-1}C_{-1})$ we are in fact finding a certain optimal point on the classifier ROC curve that corresponds to $(\lambda_1,\lambda_2)=(P_1C_1 , P_{-1}C_{-1})$. We use the term minimum risk instead of minimum cost-sensitive zero-one risk in the rest of the paper.

When the $(P_1C_1 , P_{-1}C_{-1})$ are known, we simply use them in the evaluation of the classifier as well as finding the best threshold Figure \ref{fig:roc}.

When the costs or priors of a problem are not known, a single point on ROC curve might not be a robust performance measure for the classifier. So we evaluate the risk at all points  within a low $FP$ or low $TP$of the ROC. This is equivalent to finding the  $t$-AUC \cite{asym-svm} which evaluates the area under the ROC curve within the $1$ to $t$ true negative regions. we extend this method and propose the TP-$t$-AUC and TN-$t$-AUC to evaluate the area under the ROC curve within the $1$ to $t$ true positive and $1$ to $t$ true negative regions respectively.
In the experiments we specifically report both the TP-$t$-AUC and TN-$t$-AUC in order to demonstrate the CS-SVM's ability in learning models with both high sensitivity and high specificity.

\begin{figure*}[t]
  \centering \label{fig:roc}
    \begin{tabular}{cc}
     \includegraphics[ width=2in]{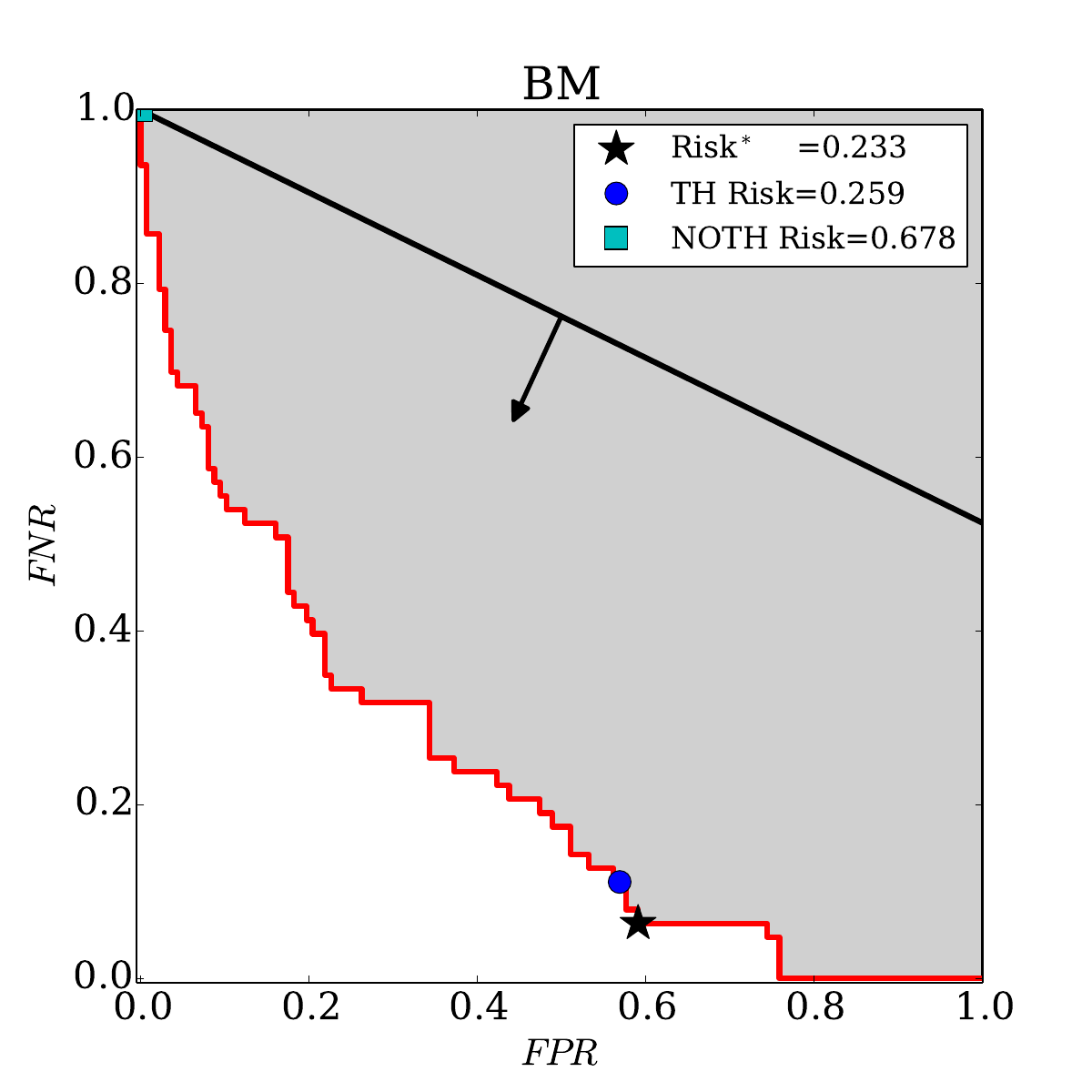}
    \includegraphics[width=2in]{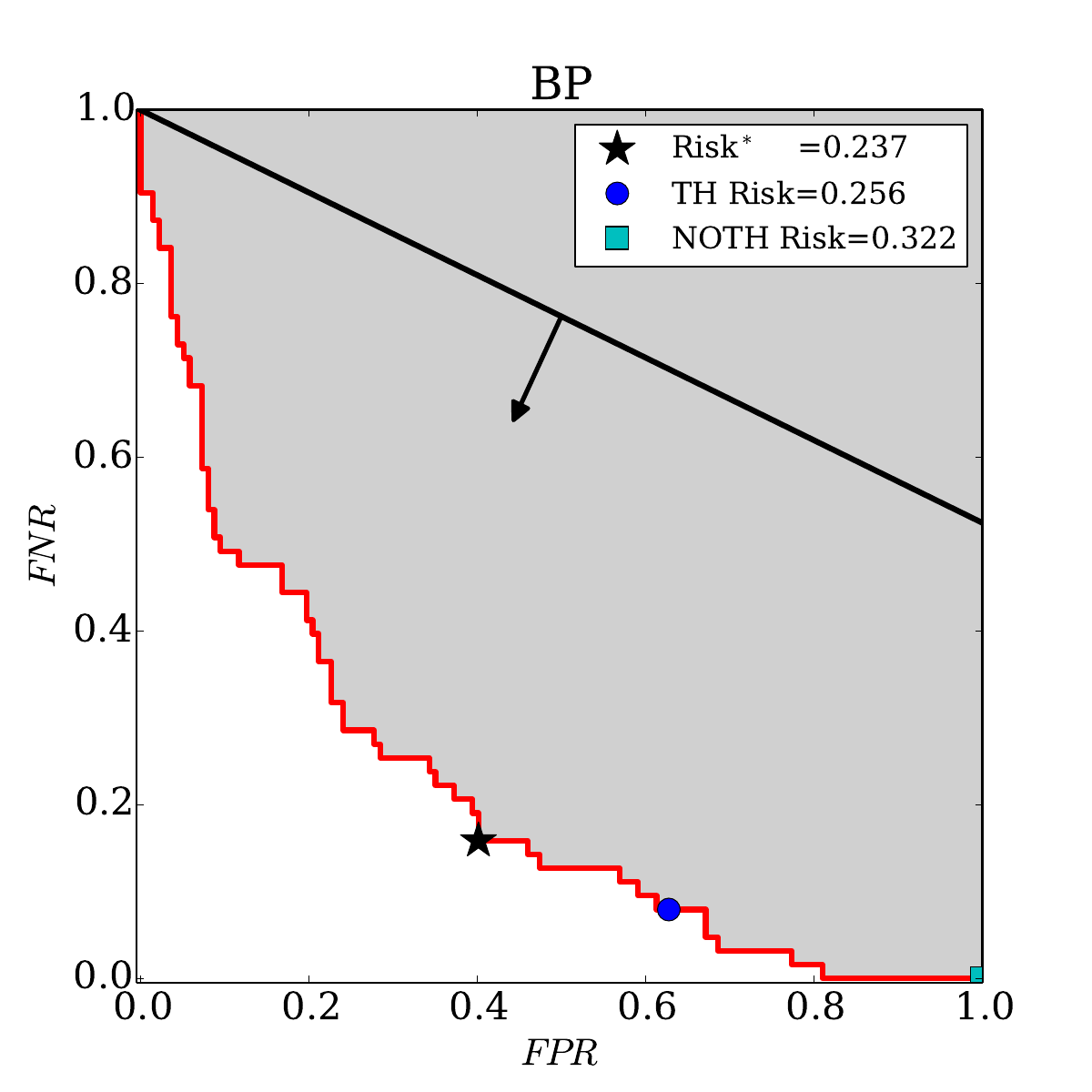}
    \includegraphics[ width=2in]{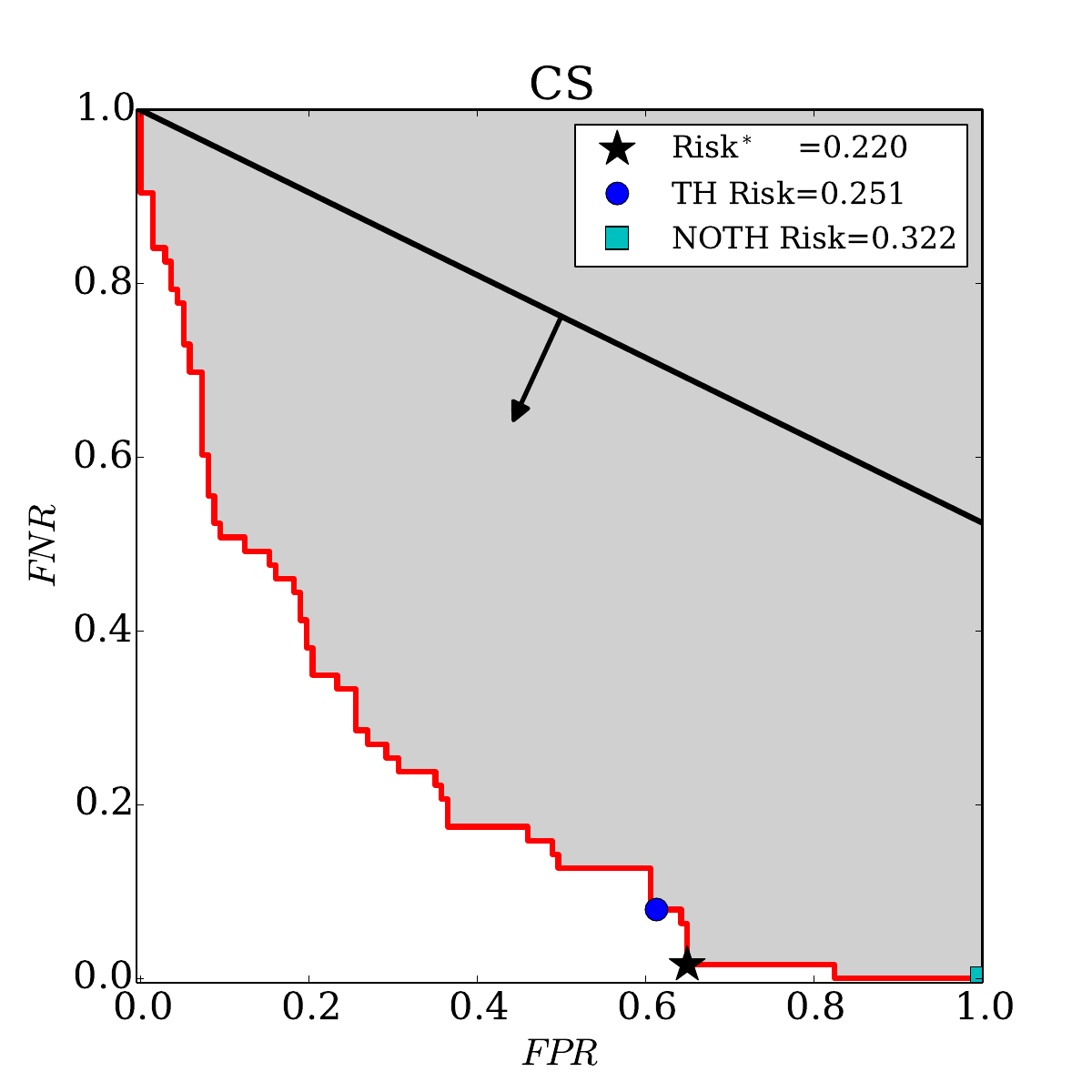}
 \end{tabular}
\caption{ROC curve of CI-SVM (left), BP-SVM (middle), and CS-SVM (right) on test examples of german dataset which costs are known. For given costs, the objective function is depicted with a black line and the best operating points is shown by Risk$^*$. Also, risks associated with the models which tuned by a threshold in the training phase is denoted by TH-Risk and risk of the models without thresholding is also shown (NOTH-Risk).}
\end{figure*}

\section{Experimental study}
In this section we conduct extensive experiments on 21 real world datasets and compare the BM-SVM, BP-SVM and CS-SVM algorithms.  The experiments are grouped into four types namely cost-sensitive learning with available class-dependent costs (CSA), cost-sensitive learning when class-dependent costs are unavailable(CSU), cost-sensitive learning with example-dependent costs (CSE) and imbalanced dataset learning(IDL). 
The datasets and experiments are further explained  in the following sections.

\subsection{Datasets}
21 datasets, created from 20 distinct datasets, are used to compare the performance of the CS-SVM algorithm with other algorithms under different scenarios.   
Table \ref{tab:datasets} shows the detailed specifications of each dataset. Each dataset is associated with a type of experiment. For example, the  KDD98 dataset is used in the CSE experiment and datasets  with large class imbalance ratios are used in IDL experiments. For each dataset we choose the class with the higher cost or fewer data points as the target or positive class. All multi-class datasets were converted to binary datasets. In particular, the binary datasets SIAM(1) and SIAM(2) are datasets which have been constructed from the same multi-class dataset but with different target class and thus different imbalance ratios. \footnote{ SIAM, Web Spam, IJCNN, MNIST, KDD99 and Covertype data sets were obtained from the LIBSVM data website.\href{http://www.csie.ntu.edu.tw/~cjlin/libsvmtools/datasets} {http://www.csie.ntu.edu.tw/~cjlin/libsvmtools/datasets} }

\begin{table*}[t]
\caption{ Specifications of the benchmark datasets. \# of Ex. is the number of example data points. \# of Feat. is the number of features. Ratio is the class imbalance ratio. Target specifies the target or positive class. Type specifies the type of experiment conducted on this dataset. } \centering \small
\begin{tabular}[t]{|l|c|c|c|c|c|}
\hline
Dataset & \# of Ex. & \# of Feat. & Ratio & Target & Type \\
\hline
German Credit & 1,000 & 24    & 1:2   & Bad (2) & CSA \\
Heart& 270   & 13    & 1:1   & Presence (2) & CSA \\
\textbf{KDD 99 (Intrusion Detection)} & 5,209,460 & 42    & 1:4   & Normal & CSA \\
\hline
\textbf{KDD 98 (Donation)} & 191,779 & 479   & 1:20  & 2     & CSE \\
\hline
Breast Cancer Diagnostic & 569   & 32    & 1:2   & Malignant (M) & CSU \\
Breast Cancer Original & 699   & 10    & 1:2   & Malignant (4) & CSU \\
Diabetes & 768   & 8     & 1:2   & Has Diabet (+1) & CSU \\
Echo-cardiogram & 132   & 12    & 1:2   & Alive (1) & CSU \\
Liver  & 345   & 6     & 1:1   & 1     & CSU \\
Sonar & 208   & 60    & 1:1   & +1    & CSU \\
Tic-Tac-Toe & 958   & 9     & 1:2   & Negative & CSU \\
\textbf{Web Spam} & 350,000 & 254   & 1:2   & -1    & CSU \\
\hline
Breast Cancer Prognostic & 198   & 34    & 1:3   & Recur ( R ) & IDL \\
\textbf{Covertype} & 581,012 & 54    & 1:211 & Cottonwood/Willow(4) & IDL \\
Hepatits & 155   & 20    & 1:4   & Die (1) & IDL \\
\textbf{IJCNN} & 141,691 & 2     & 1:10  & +1    & IDL \\
\textbf{Isolet} & 7,797 & 617   & 1:25  & K (11) & IDL \\
\textbf{MNIST} & 70,000 & 780   & 1:10  & 5     & IDL \\
\textbf{SIAM1} & 28,596 & 30438 & 1:2000 & 1,6,7,11 & IDL \\
\textbf{SIAM11} & 28,596 & 30438 & 1:716 & 11,12 & IDL \\
Survival & 306   & 3     & 1:3   & 2     & IDL \\
\hline
\end{tabular}
\label{tab:datasets}
\end{table*}

\subsection{Setup}
The RBF Gaussian kernel $k(x,x^\prime)=\exp {- \gamma{\|x-x^\prime\|}^2}$ is used for all SVM algorithms. We choose the  hyper parameters of $C$ and $\gamma$ by performing a 2D grid search and optimizing the associated performance measure (risk, TP/TN-$t$-AUC or income). Given that the size of the datasets are very different, we avoid over fitting by considering a specific search range and granularity for each dataset, but use the same range and granularity for all algorithms. In each iteration of the grid search, the performance is evaluated by 10 fold cross-validation for small datasets and evaluated on a separated validation set for large datasets which appear in bold font in Table \ref{tab:datasets}. Once the 2D grid search is complete, the hyper parameters are used to train the BM-SVM. Also, the kernel hyper parameter is used for training both the BP-SVM and the CS-SVM.

Without loss of generality, we set $C_{-1}=1$  in the BP-SVM experiments. Therefore, when considering the BP-SVM experiments we only need to perform an additional 2D grid search for $C$ and $C_{1}$. The CS-SVM actually has four independent hyper parameters, including $\gamma$.  We perform a 3D grid search on $C$, $C_1$ and $\kappa$ when the costs are not known, and a 2D search on $C$ and $\kappa$ when the costs $C_1$ and $C_{-1}$ are available. Note that in the case of available costs (CSA), setting $\kappa$ to a value other than $\kappa=\frac{1}{2C_{-1}-1}$  implicitly means that $C_{-1}$ is set to a value other than its determined value. However, we deliberately allow this in order to make use of the CS-SVM algorithm's asymmetric margin advantages. Nevertheless, we use the determined cost of $C_{-1}$ during performance evaluation. \footnote{ The source code for CS-SVM is available at \href{http://www.svcl.ucsd.edu/projects/CostLearning/}{http://www.svcl.ucsd.edu/projects/costlearning} }.
Finally, we use the  TP-0.9-AUC and TN-0.9-AUC performance measures when considering the IDL and CSU type experiments since the costs are not explicitly known in these experiments.

\subsection{Implementation}
The CS-SVM problem \eqref{eq:csdual2} is readily implemented in the dual by modifying the LibSVM \Citet{LibSVM} source code. This is done by 1) adding the regularization term to the LibSVM objective function and 2) selecting $C_1=C_1$ and $C_{-1}=\frac{1}{\kappa}$ as the cost parameters. As a result, $C$, $\gamma$, $C_1$ and $\frac{1}{\kappa}$ are the CS-SVM solver hyper parameters.

\subsection{Experiments on cost-sensitive learning with known class-dependent costs}
For these set of experiments, we compare test Risk of datasets corresponding to the point on ROC curve which determined by the threshold that is found in the training phase (Figure \ref{fig:roc}). Three datasets with known class costs are examined. Namely, the German credit card dataset \Citet{PerceptSVM,UCI}, the Statlog Heart Disease \Citet{UCI} and KDD99  \Citet{kdd99-results} datasets are considered. The minimum risk using the BM-SVM, BP-SVM and CS-SVM is shown in Table \ref{tab:risk} for each of the CSA datasets.
The CS-SVM algorithm outperforms the BP-SVM on all datasets, surpasses the BM-SVM on two and ties with the BM-SVM on one dataset.

\begin{table*}[t]
\caption{Expected risk of datasets with known class-dependent costs.} \centering \small
\begin{tabular}[t]{|l|c|c|c|}
\hline
   Dataset &    BM-SVM &    BP-SVM &    CS-SVM \\
   \hline
   German Credit &   0.26 &   0.6 & {\bf 0.25} \\
   Heart &   {\bf 0.09} & 0.1&  {\bf 0.09} \\ 
   KDD 99 &   0.054 &   0.054 & {\bf 0.045} \\    
\hline
\end{tabular}
\label{tab:risk}
\end{table*}

\subsection{Experiments on cost-sensitive learning with unknown class-dependent costs}
We consider eight datasets which do not have known costs and are not highly imbalanced. Namely, we examine the Breast Cancer Diagnostic, Breast Cancer Original, Pima Indian Diabets, Echo-cardiogram, Liver, Sonar,  Tic-Tac-Toe \Citet{UCI} and Web Spam \Citet{webspam} datasets.
The CS-SVM exhibits improved  TP-0.9-AUC (Table \ref{tab:TP-UCD}) and TN-0.9-AUC (Table \ref{tab:TN-UCD}) performance compared to BP-SVM and BM-SVM in 15 out of 16 experiments and ties in one experiment. 

\begin{table*}[t]
\caption{TP-0.9-AUC on datasets  with unknown class costs.} \centering \small
\begin{tabular}[t]{|l|c|c|c|}
\hline
   Dataset &    BM-SVM &    BP-SVM &    CS-SVM \\
    \hline
    Breast Cancer D. & 0.33  & 0.19  & \textbf{0.16} \\
    Breast Cancer O. & \textbf{0.03} & \textbf{0.03} & \textbf{0.03} \\
    Diabetes & 0.36  & 0.37  & \textbf{0.34} \\
    Echo-cardiogram & 0.43  & 0.48  & \textbf{0.35} \\
    Liver  & 0.921 & 0.921 & \textbf{0.920} \\
    Sonar & 0.40  & 0.40  & \textbf{0.38} \\
    Tic-Tac-Toe & 0.97  & 0.90  & \textbf{0.88} \\
    Web Spam & 0.03  & 0.02  & \textbf{0.01} \\
\hline
\end{tabular}
\label{tab:TP-UCD}
\end{table*}

\begin{table*}[t]
\caption{TN-0.9-AUC on datasets  with unknown class costs.} \centering \small
\begin{tabular}[t]{|l|c|c|c|}
\hline
   Dataset &    BM-SVM &    BP-SVM &    CS-SVM \\
    \hline
    Breast Cancer D. & 0.40  & 0.35  & \textbf{0.31} \\
    Breast Cancer O. & 0.17  & 0.17  & \textbf{0.16} \\
    Diabetes & 0.69  & 0.67  & \textbf{0.66} \\
    Echo-cardiogram & 0.60  & 0.60  & \textbf{0.35} \\
    Liver  & 0.90  & 0.95  & \textbf{0.88} \\
    Sonar & 0.70  & 0.62  & \textbf{0.60} \\
    Tic-Tac-Toe & 0.93  & 0.87  & \textbf{0.86} \\
    Web Spam & 0.03  & 0.03  & \textbf{0.02} \\
\hline
\end{tabular}
\label{tab:TN-UCD}
\end{table*}

\subsection{Experiments on imbalanced data learning}
We examine large datasets with severe imbalance ratios to evaluate the merit of the proposed CS-SVM algorithm on imbalanced data learning which could be the most prevailing problem in practice. The CS-SVM exhibits improved  TP-0.9-AUC (Table \ref{tab:TP-UCD}) and TN-0.9-AUC (Table \ref{tab:TN-UCD}) performance compared to BP-SVM and BM-SVM in 17 out of 18 IDL experiments and ties in one experiment.

\begin{table*}[t]
\caption{ TP-0.9-AUC on imbalanced datasets.} \centering \small
\begin{tabular}[t]{|l|c|c|c|c|}
\hline
    Dataset 			&    BM-SVM &    BP-SVM &    CS-SVM \\
    \hline
    Breast Cancer P. & 0.83  & 0.79  & \textbf{0.76} & IDL \\
    Covertype & 0.034 & 0.020 & \textbf{0.016} & IDL \\
    Hepatits & 0.56  & 0.40  & \textbf{0.36} & IDL \\
    IJCNN & 0.091 & 0.034 & \textbf{0.031} & IDL \\
    Isolet & 0.86  & 0.19  & \textbf{0.10} & IDL \\
    MNIST & 0.053 & 0.019 & \textbf{0.017} & IDL \\
    SIAM1 & 0.76  & 0.30  & \textbf{0.29} & IDL \\
    SIAM11 & \textbf{0.70} & \textbf{0.70} & \textbf{0.70} & IDL \\
    Survival & 0.89  & 0.88  & \textbf{0.87} & IDL \\
\hline
\end{tabular}
\label{tab:TP-IDL}
\end{table*}

\begin{table*}[t]
\caption{ TN-0.9-AUC on imbalanced datasets.} \centering \small
\begin{tabular}[t]{|l|c|c|c|c|}
\hline
    Dataset 			&    BM-SVM &    BP-SVM &    CS-SVM \\
    \hline
    Breast Cancer P. & 0.87  & 0.81  & \textbf{0.80} & IDL \\
    Covertype & 0.062 & 0.062 & \textbf{0.060} & IDL \\
    Hepatits & 0.70  & 0.70  & \textbf{0.67} & IDL \\
    IJCNN & 0.02  & 0.02  & \textbf{0.01} & IDL \\
    Isolet & 0.86  & 0.19  & \textbf{0.10} & IDL \\
    MNIST & 0.05  & 0.02  & \textbf{0.02} & IDL \\
    SIAM1 & 0.938 & 0.526 & \textbf{0.525} & IDL \\
    SIAM11 & 1.000 & 0.748 & \textbf{0.739} & IDL \\
    Survival & 0.66  & 0.64  & \textbf{0.63} & IDL \\
\hline
\end{tabular}
\label{tab:TN-IDL}
\end{table*}
 
\subsection{Experiments on cost-sensitive learning with example-dependent cost}
We study example-dependent cost-sensitive learning using the well known KDD98 dataset. This dataset contains information about past contributors to charities.
The task is to classify individuals as donors or non-donors for a new charity so that overall donations are maximized. The cost of sending mail and soliciting a donation is $0.68$\$ and the range of possible donations is $1-200$\$. We use the total profit performance measure ~\Citet{Elkan2001Foundations} and evaluate the algorithms according to the benefit matrix shown in Table \ref{tab:benefit}.

\begin{table*}[t] 
\caption{ Benefit matrix for the KDD98 dataset.} \centering \small
\begin{tabular}[t]{l|c|c} 
										&    Donor&    Non-donor\\
\hline   
     Predicted Donor&   ${C_{+1}}_{i}$\$ & $-0.68$\$ \\
\hline
  Predicted Non-donor&  $-{C_{+1}}_{i}$\$ &  $0$\$ \\
\end{tabular} 
\label{tab:benefit}
\end{table*}

A range of different methods and algorithms have been previously used on this dataset and some of the most profitable methods are listed in Table \ref{tab:income} and  further explained. \cite{Wong2005} proposed an ad-hoc algorithm which extracts Focused Association Rules (FAR) for the KDD98 dataset. The FAR method consist of  three subsequent algorithms of rule generating, model building and pruning and yields the best profit on the KDD98 dataset. The example dependent MetaCost (ED-MetaCost) and direct cost-sensitive method (DCSM) are both implemented by \cite{CP-unknown} and differ in  the method used for cost and probability estimation. 
Res-DIPOL and Res-ED-BP-SVM \cite{PerceptSVM} are resampling based algorithms equipped with DIPOL and ED-BP-SVM algorithms respectively. For these methods the dataset is resampled according to a modified probability distribution. 
\cite{ED-03} suggest two types of algorithms for cost sensitive learning. The first type are those that directly incorporate the costs into the learning algorithm and the second type are black box methods that convert a cost insensitive algorithm into a cost sensitive algorithm by resampling the data according to the example costs. The Polynomial kernel ED-BP-SVM (P-ED-BP-SVM) directly incorporates the costs into the learning algorithm while the proposed black box SVM (BB-CI-SVM) and black box C4.5 (BB-C4.5) are examples of the second type proposed in \cite{ED-03}.

Table \ref{tab:benefit} also shows results for the example dependent implementations of BM-SVM (ED-BM-SVM), BP-SVM (ED-BP-SVM) and CS-SVM (ED-SV-SVM) with Gaussian kernels. The ED-CS-SVM exhibits the best performance  among all ED-SVM methods.  It also ranks fifth among all methods some of which use complicated and compounded schemes.

\begin{table*}[t] 
\caption{ Income of different algorithms on the KDD98 dataset.} \centering \small
\begin{tabular}[t]{|c|c|c|c|} 
\hline
Rank &Algorithm&    Income& Comments \\ \hline
   1 & FAR & \$ 20,693  & Ad-hoc method based on sequence of three algorithms \\ \hline 
   2 & DCSM& \$ 15,329  &   Probability and cost estimation to minimize  cost \\ \hline 
   3 & BB-C4.5 & \$ 15,016  &  C4.5 on resampled dataset  \\ \hline 
   4 & KDD-Cup 98 Winner & \$ 14,712  &   Rule-based approach\\ \hline 
   5 & ED-CS-SVM &  \$14,205  &  ED-CS-SVM with Gaussian kernel $\kappa=0.97$ \\ \hline 
   6 & ED-MetaCost & \$ 14,113  &   Probability and cost estimation to minimize  cost \\ \hline 
   7 & ED-BP-SVM &  \$14,008  &  ED-BP-SVM with Gaussian kernel \\ \hline 
   8 & Res-DIPOL & \$ 14,045  &  DIPOL on resampled dataset \\ \hline
   9 & P-ED-BP-SVM & \$ 13,683  & ED-BP-SVM with Polynomial Kernel \\ \hline 
   10& BB-SVM & \$ 13,152  &  CI-SVM on resampled dataset \\ \hline 
   11& Res-ED-BP-SVM & \$ 12,883  &  ED-BP-SVM on resampled dataset \\ \hline 
   12& BM-CI-SVM & \$ 10,560  &  Standard SVM \\ \hline 
   13& Null Classifier & \$ 10,560  &  Predicts all examples as donor \\ \hline
\end{tabular}
\label{tab:income}
\end{table*}

\section{Conclusion}
In this work, we have extended the recently introduced probability 
elicitation view of loss function design to the cost sensitive 
classification problem. This extension was applied to the
SVM problem, so as to produce a cost-sensitive hinge loss function.
A cost-sensitive SVM learning algorithm was then derived, as the minimizer
of the associated risk. Unlike previous SVM algorithms, the one now
proposed enforces cost sensitivity for both separable and non-separable 
training data, enforcing a larger margin for the preferred class, independent of the choice of slack penalty. It
also offers guarantees of optimality, namely classifiers
that implement the cost-sensitive Bayes decision rule and approximate 
the cost-sensitive Bayes risk. The dual problem of CS-SVM is studied and connections between cost-sensitive learning and regularization theory and sensitivity analysis are established. Minimum expected cost-sensitive risk is considered as a metric for evaluating the performance of binary classifiers in the cost-sensitive and imbalanced data settings. The CS-SVM is also readily extended to  cost-sensitive learning with example-dependent costs. Empirical evidence confirms its superior 
performance, when compared to previous methods.

\appendix
\section{Fenchel Dual Problem} \label{app:fenchel}  
\bt[Fenchel Dual of the Regularized risk Minimization Problem] Let $\Omega: \Rbb^n \rightarrow \Rbb $ and $\phi: \Rbb \rightarrow \Rbb_+ $ be convex functions and $\dom \Omega = \dom \Phi =  \Rbb^n$, then
\beq\label{eq:fen0} 
\inf_{\beta}\{\Omega(K\beta)+\sum_i \phi(y_i K_i^T\beta)\} = \sup_{z} \{-\Omega^*(K^{-1}z)-\sum_i \phi^*(y_iz^TK^{-1}_i)\}
\eeq
which $\beta$ and $z$ are primal and dual variables, and $\Omega^*$ and $\phi^*$ are Fenchel Conjugate functions\footnote{The Fenchel conjugate of $h:\Rbb^n \rightarrow \Rbb$ is defined as
\beqq
h^*(y)=\sup_{x \in \dom{h}}\{y^Tx - h(x)\} 
\eeqq} 
of $\Omega$ and $\Phi$,respectively.
\et
\bprf
\begin{enumerate}[(i)]
	\item By the representer therem we have $f(.)=K\beta$.
	\item Fenchel Duality Theorem \cite{rockafellar70} and induction we have 
	\beq \label{eq:fen:rock}
\inf_{f}\{\Omega(f)+\sum_i \phi_i(f)\} = \sup_{g} \{-\Omega^*(g)-\sum_i \phi_i^*(g)\}
\eeq
where here $f$ and $g$ are primal and dual decision functions.
\item Let $\Phi:\Rbb^n \rightarrow \Rbb^n$, then we can write
\beq\label{eq:phiv} 
\sum_i \phi(y_i K_i^T\beta) = \bfone^T \Phi(YK\beta) 
\eeq

\item Composition with linear transformation can be conjugated by \footnote{Note that, $Y$ is a diagonal matrix with $y_{ii}\in\{-1,1\}$. So we have $Y=Y^{-1}$.} \cite{book:convex} 
\beq \label{eq:fen-comp}
\Omega(f)&=\Omega(K\beta) \Rightarrow \Omega^*(g)=\Omega^*(K^{-1}z)\\
\Phi(f)&=\Phi(YK\beta) \Rightarrow \Phi^*(g)=\Phi^*(YK^{-1}z)
\eeq
where $f$ and $g$ are primal and dual decision functions and $\beta$ and $z$ are primal and dual variables, respectively.
\item By \eqref{eq:fen-comp}, we have $g(.)=K^{-1}z$
\item From \eqref{eq:fen:rock} and \eqref{eq:phiv} we have
\beq
\bfone^T \Phi(YK^{-1}z)= \sum_i \phi^*(y_i z^TK^{-1}_i) 
\eeq

\end{enumerate}
\eprf

Conjugate of regularizer \footnote{$\Phi$ and $\Phi^*$ are both Tikhonov regularization in $\Hc$ and $\Hc^*$ with kernels $K$ and $K^{-1}$ respectively, i.e., $\Omega(f)=\frac{1}{2}\|f\|^2_\Hc=\frac{1}{2}f^TK^{-1}f$ and $\Omega^*(g)= \|g\|^2_{\Hc^*} =\frac{1}{2}z^TK^{-1}	z$} of $\Omega(K\beta)= \frac{1}{2}\beta^TK	\beta$ is given by 
\beqq
\Omega^*(K^{-1}z)=\sup_{\beta}\{z^TK^{-1}K\beta- \frac{1}{2} \beta^TK\beta\} = \frac{1}{2} z^TK^{-1}z
\eeqq
For the decision functions with a bias term, i.e. $f(x_i)=K_i^T\beta +b$, the bias is not regularized, and unregularized bias formulation introduces an equality constraint  in the dual (\Citet{rifkin07}, Section 9.1)
\beq \label{eq:eqdual}
\textbf{1}^TK^{-1}z=0
\eeq
Given $a,b\in \Rbb_{++}$, the conjugate of the Hinge loss $\phi(\bfu)=\max(b-a\bfu,0)$, can be computed 
\beqq
g^*(\bfv)=\sup_{\bfu}\{\bfu \bfv- \max(b-ax,0)\}=\begin{cases}
\underset{\bfu}{\sup}\ \{\bfu \bfv\}  \hspace{1in} & \bfu > \frac{b}{a}  \\
    \underset{\bfu}{\sup}\ \{\bfu(\bfv+a)-b\} & \bfu\le \frac{b}{a}
\end{cases}
\eeqq
which we have two cases
\begin{enumerate}[(i)]
	\item $\bfv\le0\ \Rightarrow \ g^*(\bfv)=
	\begin{cases}
\underset{\bfu >\frac{b}{a}}{\sup}\ \{\bfu \bfv\}  = \frac{b}{a}\bfv\\
    \underset{\bfu\le1}{\sup}\ \{\bfu(\bfv+a)-b\}=\left\{ \begin{array}{ll}
         \frac{b}{a}\bfv & -a\le \bfv \le 0\\
        \infty \ \ \ & \bfv<-a\end{array} \right. 
\end{cases}$
\item $\bfv>0\ \Rightarrow \ g^*(\bfv)=
	\begin{cases}
\underset{\bfu >\frac{b}{a}}{\sup}\ \{\bfu \bfv\}  = \infty\\
    \underset{\bfu\le \frac{b}{a}}{\sup}\ \{\bfu(\bfv+a)-b\}=\frac{b}{a}\bfv    
\end{cases}$
\end{enumerate} 
Thus for all $\bfu$, we can write $g^*(\bfv)=I_{[-1,0]}(\bfv)+\bfv$ or equivalently \footnote{This equivalence is legitimate because $\Omega$ is an even function.} $g^*(\bfv)=I_{[0,1]}(\bfv)-\bfv$. Now we derive the conjugate of CI-Hinge, BP-Hinge and CS-Hinge losses specifically:
\paragraph{CI-Hinge}

\beqq
 \phi(\bfu)=C\max(1-\bfu,0)=\max(C-Cu,0)\ \Rightarrow \ \phi^*(\bfv)=I_{[0,C]}(\bfv)-\bfv
\eeqq
\paragraph{BP-Hinge}
\beqq
 {\phibp}_+(\bfu)&=CC_{1}\max(1-\bfu,0)\ \Rightarrow \ {\phibp^*}_+(\bfv)=I_{[0,CC_{1}]}(\bfv)-\bfv\\
 {\phibp}_-(\bfu)&=CC_{-1}\max(1-\bfu,0)\ \Rightarrow \ {\phibp^*}_-(\bfv)=I_{[0,CC_{-1}]}(\bfv)-\bfv
\eeqq

\paragraph{CS-Hinge}
\beqq
 {\phics}_+(\bfu)&=CC_{1}\max(1-\bfu,0)\ \Rightarrow \ {\phics^*}_+(\bfv)=I_{[0,CC_{1}]}(\bfv)-\bfv\\
 {\phics}_-(\bfu)&=C\max(1- \frac{\bfu}{\kappa} ,0)\ \Rightarrow \ {\phics}^*_+(\bfv)=I_{[0,\frac{C}{\kappa}]}(\bfv)- \kappa \bfv
\eeqq
Moreover, since $\bfv\ge0$, for $C_{-1}=\frac{1}{\kappa}$we can write
\beqq
 \phics^*(\bfv)=\phibp^*(\bfv) + (1-\kappa)|\bfv|
\eeqq
and in general we have
\beqq
 \sum_i \phics^*(y_i z^T K_i^{-1})=\sum_i \phibp^*(y_i z^T K_i^{-1}) + (1-\kappa)\|  K^{-1}z^-\|_1
\eeqq
\bibliography{library}

\bibliographystyle{plainnat}
  
\end{document}